\documentclass[letterpaper,final]{article}
\usepackage{aaai}
\usepackage{times}
\usepackage[scaled]{helvet}
\usepackage{courier}
\usepackage{amssymb}
\usepackage{amsthm}
\usepackage{amsmath}
\usepackage[final]{graphicx}
\usepackage{xspace}
\usepackage{paralist}
\usepackage{multirow}
\usepackage[draft]{ifdraft}
\usepackage{bussproofs}
\usepackage[table]{xcolor}
\usepackage{thmtools}
\usepackage{thm-restate}
\usepackage{url}
\usepackage{macros}
\allowdisplaybreaks

\title{Introducing Nominals to the Combined Query Answering Approaches for $\mathcal{EL}$}
\author{Giorgio Stefanoni \and Boris Motik \and Ian Horrocks\\
Department of Computer Science, University of Oxford\\
Wolfson Building, Parks Road,\\
Oxford, OX1 3QD, UK}

\begin{document}

\maketitle

\begin{abstract}
So-called \emph{combined approaches} answer a conjunctive query over a
description logic ontology in three steps: first, they materialise certain
consequences of the ontology and the data; second, they evaluate the query
over the data; and third, they filter the result of the second phase to
eliminate unsound answers. Such approaches were developed for various members
of the DL-Lite and the $\mathcal{EL}$ families of languages, but none of them
can handle ontologies containing nominals. In our work, we bridge this gap and
present a combined query answering approach for \elho---a logic that contains
all features of the OWL 2 EL standard apart from transitive roles and complex
role inclusions. This extension is nontrivial because nominals require
equality reasoning, which introduces complexity into the first and the third
step. Our empirical evaluation suggests that our technique is suitable for
practical application, and so it provides a practical basis for conjunctive
query answering in a large fragment of OWL 2 EL.
\end{abstract}

\section{Introduction}

Description logics (DLs) \cite{DLHB-2007-ed} are a family of knowledge
representation formalisms that underpin OWL 2 \cite{ghmppss08next-steps}---an
ontology language used in advanced information systems with many practical
applications. Answering conjunctive queries (CQs) over ontology-enriched data
sets is a core reasoning service in such systems, so the computational aspects
of this problem have received a lot of interest lately. For expressive DLs,
the problem is at least doubly exponential in query size \cite{GHLS08a}. The
problem, however, becomes easier for the $\mathcal{EL}$ \cite{babl05} and the
DL-Lite \cite{Calvanese2007} families of DLs, which provide the foundation for
the OWL 2 EL and the OWL 2 QL profiles of OWL 2. An important goal of this
research was to devise not only \emph{worst-case optimal}, but also
\emph{practical} algorithms. The known approaches can be broadly classified as
follows.

The first group consists of automata-based approaches for DLs such as OWL 2 EL
\cite{KRH:elcq07} and Horn-$\mathcal{SHOIQ}$ and Horn-$\mathcal{SROIQ}$
\cite{DBLP:conf/ijcai/OrtizRS11}. While worst-case optimal, these approaches
are typically not suitable for practice since their best-case and worst-case
performance often coincide.

The second group consists of rewriting-based approaches. Roughly speaking,
these approaches rewrite the ontology and/or the query into another formalism,
typically a union of conjunctive queries or a datalog program; the relevant
answers can then be obtained by evaluating the rewriting over the data.
Rewriting-based approaches were developed for members of the DL-Lite family
\cite{Calvanese2007,ACKZ09}, and the DLs $\mathcal{ELHIO}_\bot$
\cite{DBLP:journals/japll/Perez-UrbinaMH10} and Horn-$\mathcal{SHIQ}$
\cite{DBLP:conf/aaai/EiterOSTX12}, to name just a few. A common problem,
however, is that rewritings can be exponential in the ontology and/or query
size. Although this is often not a problem in practice, such approaches are
not worst-case optimal. An exception is the algorithm by
\citeA{DBLP:conf/dlog/Rosati07} that rewrites an $\mathcal{ELH}_\bot$ ontology
into a datalog program of polynomial size; however, the algorithm also uses a
nondeterministic step to transform the CQ into a tree-shaped one, and it is
not clear how to implement this step in a goal-directed manner.

The third group consists of \emph{combined approaches}, which use a three-step
process: first, they augment the data with certain consequences of the
ontology; second, they evaluate the CQ over the augmented data; and third,
they filter the result of the second phase to eliminate unsound answers. The
third step is necessary because, to ensure termination, the first step is
unsound and may introduce facts that do not follow from the ontology; however,
this is done in a way that makes the third step feasible. Such approaches have
been developed for logics in the DL-Lite \cite{KoLuToWoZa-IJCAI11} and the
$\mathcal{EL}$ \cite{DBLP:conf/ijcai/LutzTW09} families, and they are
appealing because they are worst-case optimal and practical: only the second
step is intractable (in query size), but it can be solved using well-known
database techniques.

None of the combined approaches proposed thus far, however, handles
\emph{nominals}---concepts containing precisely one individual. Nominals are
included in OWL 2 EL, and they are often used to state that all instances of a
class have a certain property value, such as `the sex of all men is male', or
`each German city is located in Germany'. In this paper we present a combined
approach for \elho---the DL that covers all features of OWL 2 EL apart from
transitive roles and complex role inclusions. To the best of our knowledge,
this is the first combined approach that handles nominals. Our extension is
nontrivial because nominals require equality reasoning, which increases the
complexity of the first and the third step of the algorithm. In particular,
nominals may introduce recursive dependencies in the filtering conditions used
in the third phase; this is in contrast to the known combined approach for
$\mathcal{EL}$ \cite{DBLP:conf/ijcai/LutzTW09} in which filtering conditions
are not recursive and can be incorporated into the input query. To solve this
problem, our algorithm evaluates the original CQ and then uses a polynomial
function to check the relevant conditions for each answer.

Following \citeA{KRH:ELP-08}, instead of directly materialising the relevant
consequences of the ontology and the data, we transform the ontology into a
datalog program that captures the relevant consequences. Although seemingly
just a stylistic issue, a datalog-based specification may be beneficial in
practice: one can either materialise all consequences of the program bottom-up
in advance, or one can use a top-down technique to compute only the
consequences relevant for the query at hand. The latter can be particularly
useful in informations systems that have read-only access to the data, or
where data changes frequently.

We have implemented a prototypical system using our algorithm, and we carried
out a preliminary empirical evaluation of (\emph{i})~the blowup in the number
of facts introduced by the datalog program, and (\emph{ii})~the number of
unsound answers obtained in the second phase. Our experiments show both of
these numbers to be manageable in typical cases, suggesting that our algorithm
provides a practical basis for answering CQs in an expressive fragment of OWL
2 EL.

The proofs of our technical results are provided in \ifdraft{this paper's
appendix}{the technical report \cite{DBLP:journals/corr/abs-1303-7430}}.

\section{Preliminaries}

\textbf{Logic Programming.}\;\;
We use the standard notions of variables, constants, function symbols, terms,
atoms, formulas, and sentences~\cite{Fitting:1996:FLA:230183}. We often
identify a conjunction with the set of its conjuncts. A substitution $\sigma$
is a partial mapping of variables to terms; $\dom{\sigma}$ and $\rng{\sigma}$
are the domain and the range of $\sigma$, respectively; $\sproj{\sigma}{S}$ is
the restriction of $\sigma$ to a set of variables $S$; and, for $\alpha$ a
term or a formula, $\sigma(\alpha)$ is the result of simultaneously replacing
each free variable $x$ occurring in $\alpha$ with $\sigma(x)$. A \emph{Horn
clause} $C$ is an expression of the form ${B_1 \wedge \ldots \wedge B_m
\rightarrow H}$, where $H$ and each $B_i$ are atoms. Such $C$ is a \emph{fact}
if ${m = 0}$, and it is commonly written as $H$. Furthermore, $C$ is
\emph{safe} if each variable occurring in $H$ also occurs in some $B_i$. A
\emph{logic program} $\Sigma$ is a finite set of safe Horn clauses;
furthermore, $\Sigma$ is a \emph{datalog program} if each clause in $\Sigma$
is function-free.

In this paper, we interpret a logic program $\Sigma$ in a model that can be
constructed bottom-up. The \emph{Herbrand universe} of $\Sigma$ is the set of
all terms built from the constants and the function symbols occurring in
$\Sigma$. Given an arbitrary set of facts $B$, let $\Sigma(B)$ be the smallest
superset of $B$ such that, for each clause ${\varphi \rightarrow \psi \in
\Sigma}$ and each substitution $\sigma$ mapping the variables occurring in the
clause to the Herbrand universe of $\Sigma$, if ${\sigma(\varphi) \subseteq
B}$, then ${\sigma(\psi) \subseteq \Sigma(B})$. Let $I_0$ be the set of all
facts occurring in $\Sigma$; for each ${i \in \nat}$, let ${I_{i+1} =
\Sigma(I_i)}$; and let ${I = \bigcup_{i \in \nat} I_i}$. Then $I$ is the
\emph{minimal Herbrand model} of $\Sigma$, and it is well known that $I$
satisfies ${\forall \vec{x}. C}$ for each Horn clause ${C \in \Sigma}$ and
$\vec{x}$ the vector of all variables occurring in $C$.

In this paper we allow a logic program $\Sigma$ to contain the equality
predicate $\approx$. In first-order logic, $\approx$ is usually interpreted as
the identity over the interpretation domain; however, $\approx$ can also be
explicitly axiomatised \cite{Fitting:1996:FLA:230183}. Let $\Sigma_\approx$ be
the set containing clauses \eqref{eq:ref}--\eqref{eq:tra}, an instance of
clause \eqref{eq:prep} for each $n$-ary predicate $R$ occurring in $\Sigma$
and each ${1 \leq i \leq n}$, and an instance of clause \eqref{eq:frep} for
each $n$-ary function symbol $f$ occurring in $\Sigma$ and each ${1 \leq i
\leq n}$.
{\footnotesize{    
\begin{align}
    \label{eq:ref}                                          & \rightarrow x \approx x \\
    \label{eq:sym}  x_1 \approx x_2                         & \rightarrow x_2 \approx x_1 \\
    \label{eq:tra}  x_1 \approx x_2 \wedge x_2 \approx x_3  & \rightarrow x_1 \approx x_3 \\
    \label{eq:prep} R(\vec{x})\wedge x_i \approx x_i'       & \rightarrow R(x_1,\ldots,x_i',\ldots,x_n) \\
    \label{eq:frep} x_i \approx x_i'                        & \rightarrow f(\ldots,x_i,\ldots) \approx  f(\ldots,x_i',\ldots)
\end{align}
}}
The minimal Herbrand model of a logic program $\Sigma$ that contains $\approx$
is the minimal Herbrand model of ${\Sigma \cup \Sigma_\approx}$.

\textbf{Conjunctive Queries.}\;\;
A \emph{conjunctive query} (CQ) is a formula ${q = \exists \vec{y}.
\psi(\vec{x},\vec{y})}$ with $\psi$ a conjunction of function-free atoms over
variables ${\vec{x} \cup \vec{y}}$. Variables $\vec{x}$ are the \emph{answer
variables} of $q$. Let $\terms{q}$ be the set of terms occurring in $q$.

Let $\tau$ be a substitution such that $\rng{\tau}$ contains only constants.
Then, ${\tau(q) = \exists \vec{z}. \tau(\psi)}$, where $\vec{z}$ is obtained
from $\vec{y}$ by removing each variable ${y \in \vec{y}}$ for which $\tau(y)$
is defined. Note that, according to this definition, non-free variables can
also be replaced; for example, given ${q = \exists y_1,y_2. R(y_1,y_2)}$ and
${\tau = \setof{y_2 \mapsto a}}$, we have ${\tau(q) = \exists y_1. R(y_1,a)}$.

Let $\Sigma$ be a logic program, let $I$ be the minimal Herbrand model of
$\Sigma$, and let ${q = \exists \vec{y}. \psi(\vec{x},\vec{y})}$ be a CQ that
uses only the predicates occurring in $\Sigma$. A substitution $\pi$ is a
\emph{candidate answer} for $q$ in $\Sigma$ if ${\dom{\pi} = \vec{x}}$ and
$\rng{\pi}$ contains only constants; furthermore, such a $\pi$ is a
\emph{certain answer} to $q$ over $\Sigma$, written ${\Sigma \models \pi(q)}$,
if a substitution $\tau$ exists such that ${\dom{\tau} = \vec{x} \cup
\vec{y}}$, ${\pi = \sproj{\tau}{\vec{x}}}$, and ${\tau(q) \subseteq I}$.

\begin{table}[tb]
    \centering
    \footnotesize
    \begin{tabular}{r|lcr}
        \textbf{Type}       & \textbf{Axiom}                            &                               & \textbf{Clause} \\
        \hline
        1                   & $\setof{a} \ISA A$                        & $\leadsto$                    & ${A}(a)$ \\[0.7ex]
        2                   & $A \ISA B$                                & $\leadsto$                    & ${A}(x) \rightarrow B(x)$ \\[0.7ex]
        3                   & $A \ISA \setof{a}$                        & $\leadsto$                    & ${A}(x) \rightarrow x \approx a$ \\[0.7ex]
        4                   & $A_1 \sqcap A_2 \ISA A$                   & $\leadsto$                    & ${A_1}(x) \wedge {A_2}(x) \rightarrow A(x)$ \\[0.7ex]
        5                   & $\SOME{R}{A_1} \ISA A$                    & $\leadsto$                    & $R(x,y) \wedge {A_1}(y) \rightarrow {A}(x)$ \\[0.7ex]
        \multirow{2}{*}{6}  &\multirow{2}{*}{$A_1 \ISA \SOME{R}{A}$}    &\multirow{2}{*}{$\leadsto$}    & ${A_1}(x) \rightarrow R(x,f_{R,A}(x))$ \\
                            &                                           &                               & ${A_1}(x) \rightarrow {A}(f_{R,A}(x))$ \\[0.7ex]
        7                   & $R \ISA S$                                & $\leadsto$                    & $R(x,y) \rightarrow S(x,y)$ \\[0.7ex]
        8                   & $\range{R}{A}$                            & $\leadsto$                    & $R(x,y) \rightarrow A(y)$ \\
        \hline
    \end{tabular}
    \setlength{\abovecaptionskip}{5pt}
    \setlength{\belowcaptionskip}{-15pt}
    \caption{Transforming \elho Axioms into Horn Clauses}\label{table:Xi}
\end{table}

\textbf{Description Logic.}\;\;
DL \elho is defined w.r.t.\ a signature consisting of mutually disjoint and
countably infinite sets \conceptnames, \rolenames, and \indnames of
\emph{atomic concepts} (i.e., unary predicates), \emph{roles} (i.e., binary
predicates), and \emph{individuals} (i.e., constants), respectively.
Furthermore, for each individual $a\in\indnames$, expression $\setof{a}$
denotes a \emph{nominal}---that is, a concept containing precisely the
individual $a$. Also, we assume that $\top$ and $\bot$ are unary predicates
(without any predefined meaning) not occurring in $\conceptnames$. We consider
only \emph{normalised} knowledge bases, as it is well known \cite{babl05} that
each \elho knowledge base can be normalised in polynomial time without
affecting the answers to CQs. An \elho \emph{TBox} is a finite set of axioms
of the form shown in the left-hand side of Table \ref{table:Xi}, where
${A_{(i)} \in \conceptnames \cup \setof{\top}}$, ${B \in \conceptnames \cup
\setof{\top,\bot}}$, ${R,S \in \rolenames}$, and ${a \in \indnames}$. An
\emph{ABox} \A is a finite set of facts constructed using the symbols from
${\conceptnames \cup \setof{\top, \bot}}$, $\rolenames$, and $\indnames$.
Finally, an \elho \emph{knowledge base} (KB) is a tuple ${\K =
\tuple{\T,\A}}$, where \T is an \elho TBox \T and an \A is an ABox such that
each predicate occurring in \A also occurs in \T.

We interpret \K as a logic program. Table \ref{table:Xi} shows how to
translate a TBox \T into a logic program $\Xi(\T)$. Moreover, let $\top(\T)$
be the set of the following clauses instantiated for each atomic concept $A$
and each role $R$ occurring in \T.
\begin{displaymath}
    A(x) \rightarrow \top(x)    \quad   R(x,y) \rightarrow \top(x)  \quad   R(x,y) \rightarrow \top(y)
\end{displaymath}
A knowledge base ${\K = \tuple{\T,\A}}$ is translated into the logic program
${\Xi(\K) = \Xi(\T) \cup \top(\T) \cup \A}$. Then, \K is \emph{unsatisfiable}
if ${\Xi(\K) \models \exists y. \bot(y)}$. Furthermore, given a conjunctive
query $q$ and a candidate answer $\pi$ for $q$, we write ${\K \models \pi(q)}$
iff \K is unsatisfiable or ${\Xi(\K) \models \pi(q)}$. Although somewhat
nonstandard, our definitions of DLs are equivalent to the ones based on the
standard denotational semantics \cite{DLHB-2007-ed}. Given a candidate answer
$\pi$ for $q$, deciding whether ${\Xi(\K) \models \pi(q)}$ holds is
\np-complete in \emph{combined complexity}, and \ptime-complete in \emph{data
complexity} \cite{KRH:elcq07}.

\section{Datalog Rewriting of \elho TBoxes}

For the rest of this section, we fix an arbitrary \elho knowledge base ${\K =
\tuple{\T,\A}}$. We next show how to transform \K into a datalog program
$\dat(\K)$ that can be used to check the satisfiability of \K. In the
following section, we then show how to use $\dat(\K)$ to answer conjunctive
queries.

Due to axioms of type 6 (cf.\ Table \ref{table:Xi}), $\Xi(\K)$ may contain
function symbols and is generally not a datalog program; thus, the evaluation
of $\Xi(\K)$ may not terminate. To ensure termination, we eliminate function
symbols from $\Xi(\K)$ using the technique by \citeA{KRH:ELP-08}: for each ${A
\in \conceptnames \cup \setof{\top}}$ and each ${R \in \rolenames}$ occurring
in \T, we introduce a globally fresh and unique \emph{auxiliary individual}
$o_{R,A}$. Intuitively, $o_{R,A}$ represents all terms in the Herbrand
universe of $\Xi(\K)$ needed to satisfy the existential concept $\SOME{R}{A}$.
\citeA{KRH:ELP-08} used this technique to facilitate taxonomic reasoning,
while we use it to obtain a practical CQ answering algorithm. Please note that
$o_{R,A}$ depends on both $R$ and $A$, whereas in the known approaches such
individuals depend only on $A$ \cite{DBLP:conf/ijcai/LutzTW09} or $R$
\cite{KoLuToWoZa-IJCAI11}.

\begin{definition}
    Datalog program $\dat(\T)$ is obtained by translating each axiom of type
    other than 6 in the TBox $\T$ of $\K$ into a clause as shown in Table
    \ref{table:Xi}, and by translating each axiom ${A_1 \ISA \SOME{R}{A}}$ in
    $\T$ into clauses ${A_1(x) \rightarrow R(x,o_{R,A})}$ and ${A_1(x)
    \rightarrow A(o_{R,A})}$. Furthermore, the translation of \K into datalog
    is given by ${\dat(\K) = \dat(\T) \cup \top(\T) \cup \A}$.
\end{definition}

\begin{example}\label{example:translation}
Let $\T$ be the following \elho TBox:
\begin{displaymath}
\renewcommand{\arraystretch}{1.1}
\begin{array}{r@{\;}lr@{\;}l}
    \mathit{KRC}               & \ISA \SOME{\mathit{taught}}{\mathit{JProf}}   & \SOME{\mathit{taught}}{\top}    & \ISA \mathit{Course}  \\
    \mathit{Course}            & \ISA \SOME{\mathit{taught}}{\mathit{Prof}}    & \setof{\mathit{kr}}              & \ISA \mathit{KRC} \\
    \mathit{Prof}              & \ISA \SOME{\mathit{advisor}}{\mathit{Prof}}   & \mathit{KRC}                     & \ISA \mathit{Course} \\
    \mathit{JProf}             & \ISA \setof{\mathit{john}}                    & \multicolumn{2}{@{}l@{}}{$\range{\mathit{taught}}{\mathit{Prof}}$} \\
\end{array}
\end{displaymath}
Then, $\dat(\T)$ contains the following clauses:
\begin{displaymath}
\renewcommand{\arraystretch}{1.1}
\setlength{\arraycolsep}{2pt}
\begin{array}{@{}l@{\;\;}l@{}}
    \mathit{KRC}(x) \rightarrow \mathit{taught}(x,o_{T,J})      & \mathit{JProf}(x) \rightarrow x \approx \mathit{john}\\
    \mathit{KRC}(x) \rightarrow \mathit{JProf}(o_{T,J})         & \mathit{taught}(x,y) \rightarrow \mathit{Course}(x) \\
    \mathit{Course}(x) \rightarrow \mathit{taught}(x,o_{T,P})   & \mathit{KRC}(kr) \\
    \mathit{Course}(x) \rightarrow \mathit{Prof}(o_{T,P})       & \mathit{KRC}(x)  \rightarrow \mathit{Course}(x) \\
    \mathit{Prof}(x) \rightarrow \mathit{advisor}(x,o_{A,P})    & \mathit{taught}(x,y) \rightarrow \mathit{Prof}(y)\\
    \mathit{Prof}(x) \rightarrow \mathit{Prof}(o_{A,P})         & \phantom{;}\qquad\qquad\qquad\qquad\quad\;\;\fulld \\
\end{array}
\end{displaymath}
\end{example}

\begin{figure}[t]
    \centering
    \includegraphics[scale=0.6]{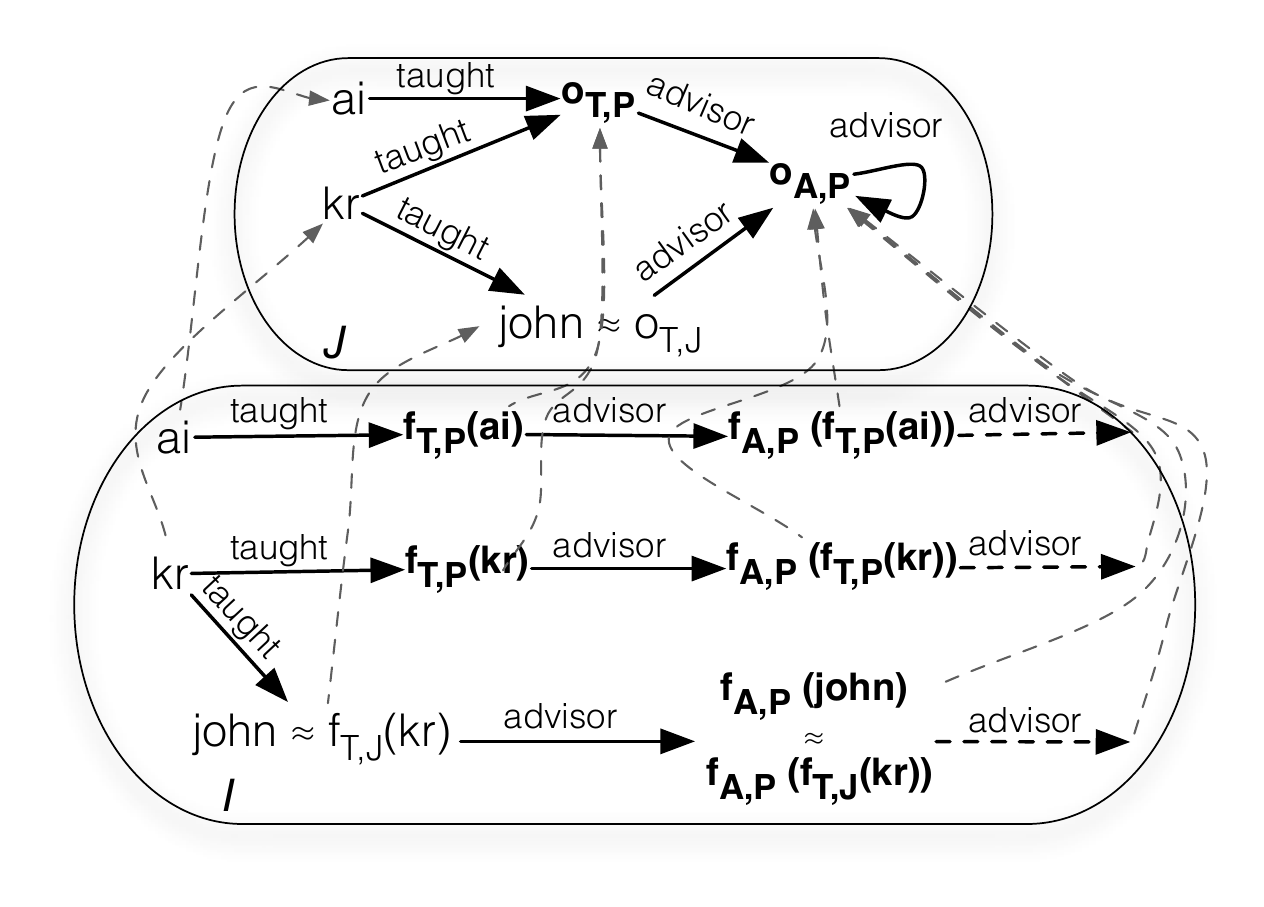}
    \setlength{\abovecaptionskip}{5pt}
    \setlength{\belowcaptionskip}{-10pt}
    \caption{Representing the Models of $\Xi(\K)$. }\label{figure:model}
\end{figure}

The following result straightforwardly follows from the definition of
$\Xi(\K)$ and $\dat(\K)$.

\begin{restatable}{proposition}{propprogsize}\label{prop:prog-size}
    Program $\dat(\K)$ can be computed in time linear in the size of \K.
\end{restatable}

Next, we prove that the datalog program $\dat(\K)$ can be used to decide the
satisfiability of $\K$. To this end, we define a function $\delta$ that maps
each term $w$ in the Herbrand universe of $\Xi(\K)$ to the Herbrand universe
of $\dat(\K)$ as follows:
\begin{displaymath}
    \delta(w) = \begin{cases}
                    w       & \text{ if } w \in \indnames, \\
                    o_{R,A} & \text{ if $w$ is of the form } w=f_{R,A}(w'). \\
                \end{cases}
\end{displaymath}
Let \lpmodel and \datalogmodel be the minimal Herbrand models of $\Xi(\K)$ and
$\dat(\K)$, respectively. Mapping $\delta$ establishes a tight relationship
between \lpmodel and \datalogmodel as illustrated in the following example.

\begin{example}\label{example:models}
Let ${\A = \setof{\mathit{Course(ai)}}}$, let $\T$ be as in
Example~\ref{example:translation}, and let ${\K = \tuple{\T,\A}}$. Figure
\ref{figure:model} shows a graphical representation of the minimal Herbrand
models \lpmodel and \datalogmodel of $\Xi(\K)$ and $\dat(\K)$, respectively.
The grey dotted lines show how $\delta$ relates the terms in $\lpmodel$ to the
terms in $\datalogmodel$. For the sake of clarity, Figure \ref{figure:model}
does not show the reflexivity of $\approx$. \fulld
\end{example}

Mapping $\delta$ is a homomorphism from \lpmodel to \datalogmodel.

\begin{restatable}{lemma}{lemmahomomorphism}
    \label{lemma:homomorphism}
    Let \lpmodel and \datalogmodel be the minimal Herbrand models of $\Xi(\K)$
    and $\dat(\K)$, respectively. Mapping $\delta$ satisfies the following
    three properties for all terms $w'$ and $w$, each ${B \in \conceptnames
    \cup \setof{\top,\bot}}$, and each ${R \in \rolenames}$.
    \begin{enumerate}
        \item ${B(w) \in \lpmodel}$ implies ${B(\delta(w)) \in
        \datalogmodel}$.

        \item ${R(w',w) \in \lpmodel}$ implies ${R(\delta(w'),\delta(w)) \in
        \datalogmodel}$.

        \item ${w' \approx w \in \lpmodel}$ implies ${\delta(w') \approx
        \delta(w) \in \datalogmodel}$.
    \end{enumerate}
\end{restatable}

For a similar result in the other direction, we need a couple of definitions.
Let $H$ be an arbitrary Herbrand model. Then, $\restricteddomain{H}$ is the
set containing each term $w$ that occurs in $H$ in at least one fact with a
predicate in ${\conceptnames \cup \setof{\top,\bot} \cup \rolenames}$; note
that, by this definition, we have ${w \not\in \restricteddomain{H}}$ whenever
$w$ occurs in $H$ only in assertions involving the $\approx$ predicate.
Furthermore, \aux{H} is the set of all terms ${w \in \restricteddomain{H}}$
such that, for each term $w'$ with ${w \approx w' \in H}$, we have ${w' \not
\in \indnames}$. We say that the terms in $\aux{H}$ are `true' auxiliary
terms---that is, they are not equal to an individual in $\indnames$. In Figure
\ref{figure:model}, bold terms are `true' auxiliary terms in \lpmodel and
\datalogmodel.

\begin{restatable}{lemma}{lemmadatembed}
    \label{lemma:dat-embed}
    Let \lpmodel and \datalogmodel be the minimal Herbrand models of $\Xi(\K)$
    and $\dat(\K)$, respectively. Mapping $\delta$ satisfies the following
    five properties for all terms $w_1$ and $w_2$ in
    $\restricteddomain{\lpmodel}$, each ${B \in \conceptnames \cup
    \setof{\top,\bot}}$, and each ${R \in \rolenames}$.
    \begin{enumerate}
        \item ${B(\delta(w_1)) \in \datalogmodel}$ implies that ${B(w_1) \in
        \lpmodel}$.
        
        \item ${R(\delta(w_1),\delta(w_2)) \in \datalogmodel}$ and
        ${\delta(w_2) \not\in \aux{\datalogmodel}}$ imply that \\
        ${R(w_1,w_2) \in \lpmodel}$.
        
        \item ${R(\delta(w_1),\delta(w_2)) \in \datalogmodel}$ and
        ${\delta(w_2) \in \aux{\datalogmodel}}$ imply that \\
        $\delta(w_2)$ is of the form $o_{P,A}$, that ${R(w_1,f_{P,A}(w_1)) \in
        \lpmodel}$, and that a term $w_1'$ exists such that ${R(w_1', w_2) \in
        \lpmodel}$.

        \item ${\delta(w_1) \approx \delta(w_2) \in \datalogmodel}$ and
        ${\delta(w_2) \not\in \aux{\datalogmodel}}$ imply that \\
        ${w_1 \approx w_2 \in \lpmodel}$.
        
        \item For each term $u$ occurring in \datalogmodel, term ${w \in
        \restricteddomain{\lpmodel}}$ exists such that ${\delta(w) = u}$.
    \end{enumerate}
\end{restatable}

Lemmas \ref{lemma:homomorphism} and \ref{lemma:dat-embed} allow us to decide
the satisfiability of \K by answering a simple query over $\dat(\K)$, as shown
in Proposition \ref{prop:sat}. The complexity claim is due to the fact that
each clause in $\dat(\K)$ contains a bounded number of
variables~\cite{DBLP:journals/csur/DantsinEGV01}.

\begin{restatable}{proposition}{propsat}
    \label{prop:sat}
    For \K an arbitrary \elho knowledge base, ${\Xi(\K) \models \exists
    y.\bot(y)}$ if and only if ${\dat(\K) \models \exists y.\bot(y)}$.
    Furthermore, the satisfiability of \K can be checked in time polynomial in
    the size of \K.
\end{restatable}

\section{Answering Conjunctive Queries}

In this section, we fix a satisfiable \elho knowledge base ${\K =
\tuple{\T,\A}}$ and a conjunctive query ${q = \exists \vec{y}.
\psi(\vec{x},\vec{y})}$. Furthermore, we fix $\lpmodel$ and $\datalogmodel$ to
be the minimal Herbrand models of $\Xi(\K)$ and $\dat(\K)$, respectively.

While $\dat(\K)$ can be used to decide the satisfiability of \K, the following
example shows that $\dat(\K)$ cannot be used directly to compute the answers
to $q$.

\begin{example}\label{example:spurious}
Let \K be as in Example \ref{example:models}, and let $q_1$, $q_2$, and $q_3$
be the following conjunctive queries:
\begin{displaymath}
\begin{array}{@{}r@{\;}l@{}}
	q_1 =   & \mathit{taught}(x_1,x_2) \\[1ex]
	q_2 = 	& \exists y_1,y_2,y_3.\; \mathit{taught}(x_1,y_1) \wedge \mathit{taught}(x_2,y_2) \; \wedge \\
			& \quad \mathit{advisor}(y_1,y_3) \wedge \mathit{advisor}(y_2,y_3) \\[1ex]
	q_3 = 	& \exists y.\; \mathit{advisor}(y,y) \\
\end{array}
\end{displaymath}
Furthermore, let $\tau_i$ be the following substitutions:
\begin{displaymath}
\begin{array}{@{}r@{\;}l@{}}
	\tau_1 =    & \setof{x_1 \mapsto \mathit{kr}, \; x_2 \mapsto o_{T,P}} \\[0.7ex]
	\tau_2 =    & \setof{x_1 \mapsto \mathit{kr}, \; x_2 \mapsto \mathit{ai}, \; \\
	            & \qquad\quad y_1 \mapsto o_{T,P}, \; y_2 \mapsto o_{T,P}, \; y_3 \mapsto o_{A,P}} \\[0.7ex]
	\tau_3 =    & \setof{y \mapsto o_{A,P}} \\
\end{array}
\end{displaymath}
Finally, let each $\pi_i$ be the projection of $\tau_i$ to the answer
variables of $q_i$. Using Figure \ref{figure:model}, one can readily check
that ${\dat(\K) \models \tau_i(q_i)}$, but ${\Xi(\K) \not\models \pi_i(q_i)}$,
for each ${1 \leq i \leq 3}$. \fulld
\end{example}

This can be explained by observing that \datalogmodel is a homomorphic image
of \lpmodel. Now homomorphisms preserve CQ answers (i.e.,~${\Xi(\K) \models
\pi(q)}$ implies ${\dat(\K) \models \pi(q)}$), but they can also introduce
unsound answers (i.e., ${\dat(\K) \models \pi(q)}$ does not necessarily imply
${\Xi(\K) \models \pi(q)}$). This gives rise to the following notion of
spurious answers.

\begin{definition}
    A substitution $\tau$ with ${\dom{\tau} = \vec{x} \cup \vec{y}}$ and
    ${\dat(\K) \models \tau(q)}$ is a \emph{spurious answer} to $q$ if
    $\sproj{\tau}{\vec{x}}$ is not a certain answer to $q$ over $\Xi(\K)$.
\end{definition}

Based on these observations, we answer $q$ over \K in two steps: first, we
evaluate $q$ over $\dat(\K)$ and thus obtain an overestimation of the certain
answers to $q$ over $\Xi(\K)$; second, for each substitution $\tau$ obtained
in the first step, we eliminate spurious answers using a special function
$\isSpurious$. We next formally introduce this function. We first present all
relevant definitions, after which we discuss the intuitions. As we shall see,
each query in Example \ref{example:spurious} illustrates a distinct source of
spuriousness that our function needs to deal with.

\begin{definition}\label{def:sim}
    Let $\tau$ be a substitution s.t.\ ${\dom{\tau} = \vec{x} \cup \vec{y}}$
    and ${\dat(\K) \models \tau(q)}$. Relation ${\sim \; \subseteq \terms{q}
    \times \terms{q}}$ for $q$, $\tau$, and $\dat(\K)$ is the smallest
    reflexive, symmetric, and transitive relation closed under the
    $\mathsf{fork}$ rule, where $\aux{\dat(\K)}$ is the set containing each
    individual $u$ from $\dat(\K)$ for which no individual ${c \in \indnames}$
    exists such that ${\dat(\K) \models u \approx c}$.
    \begin{displaymath}
        \AxiomC{$s'\sim t'$}
        \LeftLabel{$\mathsf{(fork)}$}
        \RightLabel{\begin{tabular}{@{\;}l@{}}
                        \footnotesize{$R(s,s')$ and $P(t,t')$ occur in $q$, and} \\
                        \quad \footnotesize{$\tau(s') \in \aux{\dat(\K)}$}
                    \end{tabular}}
        \UnaryInfC{$s \sim t$}
        \DisplayProof
    \end{displaymath}
\end{definition}

Please note that the definition $\aux{\dat(\K)}$ is actually a reformulation
of the definition of $\aux{\datalogmodel}$, but based on the consequences of
$\dat(\K)$ rather than the facts in \datalogmodel.

Relation $\sim$ is reflexive, symmetric, and transitive, so it is an
equivalence relation, which allows us to normalise each term ${t \in
\terms{q}}$ to a representative of its equivalence class using the mapping
$\gamma$ defined below. We then construct a graph $G_\mathsf{aux}$ that checks
whether substitution $\tau$ matches `true' auxiliary individuals in a way that
cannot be converted to a match over `true' auxiliary terms in \lpmodel.

\begin{definition}\label{def:auxgraph}
    Let $\tau$ and $\sim$ be as specified in Definition \ref{def:sim}.
    Function ${\gamma : \terms{q} \mapsto \terms{q}}$ maps each term ${t \in
    \terms{q}}$ to an arbitrary, but fixed representative $\gamma(t)$ of the
    equivalence class of $\sim$ that contains $t$. Furthermore, the directed
    graph ${G_\mathsf{aux} = \tuple{V_\mathsf{aux}, E_\mathsf{aux}}}$ is
    defined as follows.
    \begin{itemize}
        \item Set $V_\mathsf{aux}$ contains a vertex ${\gamma(t) \in
        \terms{q}}$ for each term ${t \in \terms{q}}$ such that ${\tau(t) \in
        \aux{\dat(\K)}}$.

        \item Set $E_\mathsf{aux}$ contains an edge
        $\tuple{\gamma(s),\gamma(t)}$ for each atom of the form $R(s,t)$ in
        $q$ such that ${\setof{\gamma(s),\,\gamma(t)} \subseteq
        V_\mathsf{aux}}$.
    \end{itemize}
    Query $q$ is \emph{aux-cyclic w.r.t.\ $\tau$ and \dat(\K)} if
    $G_\mathsf{aux}$ contains a cycle; otherwise, $q$ is \emph{aux-acyclic
    w.r.t.\ $\tau$ and \dat(\K)}.
\end{definition}

We are now ready to define our function that checks whether a substitution
$\tau$ is a spurious answer.

\begin{definition}\label{def:isSpurious}
    Let $\tau$ and $\sim$ be as specified in Definition \ref{def:sim}. Then,
    function $\spurious{q}{\dat(\K)}{\tau}$ returns $\true$ if and only if at
    least one of the following conditions hold.
    \begin{enumerate}[(a)]
        \item\label{spur:cond1} Variable ${x \in \vec{x}}$ exists such that
        ${\tau(x) \not\in \indnames}$.

        \item\label{spur:cond2} Terms $s$ and $t$ occurring in $q$ exist such
        that ${s \sim t}$ and ${\dat(\K) \not\models \tau(s) \approx
        \tau(t)}$.

        \item\label{spur:cond3} Query $q$ is aux-cyclic w.r.t.\ $\tau$ and
        $\dat(\K)$.
    \end{enumerate} 
\end{definition}

We next discuss the intuition behind our definitions. We ground our discussion
in minimal Herbrand models \lpmodel and \datalogmodel, but our technique does
not depend on such models: all conditions are stated as entailments that can
be checked using an arbitrary sound and complete technique. Since \K is an
\elho knowledge base, model $\lpmodel$ is \emph{forest-shaped}: roughly
speaking, the role assertions in \lpmodel that involve at least one functional
term are of the form $R(w_1,f_{R,A}(w_1))$ or $R(w_1,a)$ for ${a \in
\indnames}$; thus, \lpmodel can be viewed as a family of directed trees whose
roots are the individuals in $\indnames$ and whose edges point from parents to
children or to the individuals in $\indnames$. This is illustrated in Figure
\ref{figure:model}, whose lower part shows the the forest-model of the
knowledge base from Example \ref{example:spurious}. Note that assertions of
the form $R(w_1,a)$ are introduced via equality reasoning.

Now let $\tau$ be a substitution such that ${\dat(\K) \models \tau(q)}$, and
let ${\pi = \sproj{\tau}{\vec{x}}}$. If $\tau$ is not a spurious answer, it
should be possible to convert $\tau$ into a substitution $\pi^*$ such that
${\pi = \sproj{\pi^*}{\vec{x}}}$ and ${\pi^*(q) \subseteq \lpmodel}$. Using
the queries from Example \ref{example:spurious}, we next identify three
reasons why this may not be possible.

First, $\tau$ may map an answer variable of $q$ to an auxiliary individual, so
by the definition $\pi$ cannot be a certain answer to $q$; condition
\eqref{spur:cond1} of Definition \ref{def:isSpurious} identifies such cases.
Query $q_1$ and substitution $\tau_1$ from Example \ref{example:spurious}
illustrate such a situation: ${\tau_2(x_2) = \mathit{o_{T,P}}}$ and
$\mathit{o_{T,P}}$ is a `true' auxiliary individual, so $\pi_1$ is not a
certain answer to $q_1$.

The remaining two problems arise because model $\datalogmodel$ is not
forest-shaped, so $\tau$ might map $q$ into $\datalogmodel$ in a way that
cannot be converted into a substitution $\pi^*$ that maps $q$ into $\lpmodel$.

The second problem is best explained using substitution $\tau_2$ and query
$q_2$ from Example \ref{example:spurious}. Query $q_2$ contains a `fork'
${\mathit{advisor}(y_1,y_3) \wedge \mathit{advisor}(y_2,y_3)}$. Now
${\tau_2(y_3) = o_{A,P}}$ is a `true' auxiliary individual, and so it
represents `true' auxiliary terms $f_{A,P}(f_{T,P}(\mathit{ai}))$, \,
$f_{A,P}(f_{T,P}(\mathit{kr}))$, and so on. Since $\lpmodel$ is forest-shaped,
a match $\pi_2^*$ for $q$ in $\lpmodel$ obtained from $\tau_2$ would need to
map $y_3$ to one of these terms; let us assume that ${\pi_2^*(y_3) =
f_{A,P}(f_{T,P}(\mathit{ai}))}$. Since $\lpmodel$ is forest-shaped and
$f_{A,P}(f_{T,P}(\mathit{ai}))$ is a `true' auxiliary term, this means that
both $y_1$ and $y_2$ must be mapped to the same term (in both \datalogmodel
and \lpmodel). This is captured by the $\mathsf{(fork)}$ rule: in our example,
the rule derives ${y_1 \sim y_2}$, and condition \eqref{spur:cond2} of
Definition \ref{def:isSpurious} checks whether $\tau_2$ maps $y_1$ and $y_2$
in a way that satisfies this constraint. Note that, due to role hierarchies,
the rule needs to be applied to atoms $R(s,s')$ and $P(t,t')$ with ${R \neq
P}$. Moreover, such constraints must be propagated further up the query. In
our example, due to ${y_1 \sim y_2}$, atoms ${\mathit{taught}(x_1,y_1) \wedge
\mathit{taught}(x_2,y_2)}$ in $q_2$ also constitute a `fork', so the rule
derives ${x_1 \sim x_2}$; now this allows condition \eqref{spur:cond2} of
Definition~\ref{def:isSpurious} to correctly identify $\tau_2$ as spurious.

The third problem is best explained using substitution $\tau_3$ and query
$q_3$ from Example \ref{example:spurious}. Model \datalogmodel contains a
`loop' on individual $o_{A,P}$, which allows $\tau_3$ to map $q_3$ into
\datalogmodel. In contrast, model \lpmodel is forest-shaped, and so the `true'
auxiliary terms that correspond to $o_{A,P}$ do not form loops. Condition
\eqref{spur:cond3} of Definition \ref{def:isSpurious} detects such situations
using the graph $G_\mathsf{aux}$. The vertices of $G_\mathsf{aux}$ correspond
to the terms of $q$ that are matched to `true' auxiliary individuals (mapping
$\gamma$ simply ensures that equal terms are represented as one vertex), and
edges of $G_\mathsf{aux}$ correspond to the role atoms in $q$. Hence, if
$G_\mathsf{aux}$ is cyclic, then the substitution $\pi^*$ obtained from $\tau$
would need to match the query $q$ over a cycle of `true' auxiliary terms,
which is impossible since $\lpmodel$ is forest-shaped.

Unlike the known combined approaches, our approach does not extend $q$ with
conditions that detect spurious answers. Due to nominals, the relevant
equality constraints have a recursive nature, and they depend on both the
substitution $\tau$ and on the previously derived constraints. Consequently,
filtering in our approach is realised as postprocessing; furthermore, to
ensure correctness of our filtering condition, auxiliary individuals must
depend on both a role and an atomic concept. The following theorem proves the
correctness of our approach.

\begin{restatable}{theorem}{thmain}\label{th:main}
    Let ${\K = \tuple{\T,\A}}$ be a satisfiable \elho KB, let ${q = \exists
    \vec{y}. \psi(\vec{x},\vec{y})}$ be a CQ, and let ${\pi : \vec{x} \mapsto
    \indnames}$ be a candidate answer for $q$. Then, ${\Xi(\K) \models
    \pi(q)}$ iff a substitution $\tau$ exists such that ${\dom{\tau} = \vec{x}
    \cup \vec{y}}$, ${\sproj{\tau}{\vec{x}} = \pi}$, ${\dat(\K) \models
    \tau(q)}$, and ${\spurious{q}{\dat(\K)}{\tau} = \false}$.
\end{restatable}

Furthermore, $\spurious{q}{\dat(\K)}{\tau}$ can be evaluated in polynomial
time, so the main source of complexity in our approach is in deciding whether
${\dat(\K) \models \tau(q)}$ holds. This gives rise to the following result.

\begin{restatable}{theorem}{thupperbound}\label{th:upperbound}
    Deciding whether ${\K \models \pi(q)}$ holds can be implemented in
    nondeterministic polynomial time w.r.t.\ the size of $\K$ and $q$, and in
    polynomial time w.r.t.\ the size of $\A$.
\end{restatable}

\section{Evaluation}

\begin{table}[t]
    \centering
    \small
    \setlength{\tabcolsep}{.4em}
    \begin{tabular}[t]{|c|>{$}r<{$}>{$}l<{$}|>{$}r<{$}>{$}l<{$}|>{$}r<{$}>{$}l<{$}|}
        \hline 
                & \multicolumn{2}{c|}{\text{Individuals}}       & \multicolumn{2}{c|}{\text{Unary facts}}       & \multicolumn{2}{c|}{\text{Binary facts}} \\
                & \multicolumn{2}{c|}{(\% in \aux{\dat(\K)})}   & \multicolumn{2}{c|}{(\% over \aux{\dat(\K)})} & \multicolumn{2}{c|}{(\% over \aux{\dat(\K)})} \\
        \hline    
        L-5     & 100848        &                               & 169079        &                               & 296941        & \\
        Mat.\   & 100868        & $(0.01)$                      & 309350        & $(0.01)$                      & 632489        & $(49.2)$ \\
        \hline
        L-10    & 202387        &                               & 339746        &                               & 598695        & \\
        Mat.\   & 202407        & $(0.01)$                      & 621158        & $(0.01)$                      & 1277575       & $(49.3)$ \\
        \hline
        L-20    & 426144        &                               & 714692        &                               & 1259936       & \\
        Mat.\   & 426164        & $(0.01)$                      & 1304815       & $(0.01)$                      & 2691766       & $(49.3)$ \\
        \hline
        \hline
        SEM     & 17945         &                               & 17945         &                               & 47248         & \\
        Mat.\   & 17953         & $(0.04)$                      & 25608         & $(0.03)$                      & 76590         & $(38.3)$ \\
        \hline
    \end{tabular}
    \setlength{\abovecaptionskip}{5pt}
    \setlength{\belowcaptionskip}{-10pt}
    \caption{Size of the materialisations.}\label{table:size}
\end{table}
\begin{table*}[t]
    \centering
    \small
    \setlength{\tabcolsep}{.15em}
    \begin{tabular}[b]{|r| *{7}{>{$}r<{$}>{$}c<{$}|}}
        \hline
         LSTW   & \multicolumn{2}{>{$}c<{$}|}{q^l_1}        & \multicolumn{2}{>{$}c<{$}|}{q^l_2}        & \multicolumn{2}{>{$}c<{$}|}{q^l_3}        & \multicolumn{2}{>{$}c<{$}|}{q^l_5}        & \multicolumn{2}{>{$}c<{$}|}{q^l_8}        & \multicolumn{2}{>{$}c<{$}|}{q^l_9}        & \multicolumn{2}{>{$}c<{$}|}{q^l_{10}} \\
                & \text{Tot}    & (\%)                      & \text{Tot}    & (\%)                      & \text{Tot}    & (\%)                      & \text{Tot}    & (\%)                      & \text{Tot}    & (\%)                      & \text{Tot}    & (\%)                      & \text{Tot}    & (\%) \\
        \hline            
         L-5    & 116\text{K}   & \multirow{3}{*}{(4.0)}    & 3.7\text{M}   & \multirow{3}{*}{(100.0)}  & 10            & \multirow{3}{*}{(0.0)}    & 28\text{K}    & \multirow{3}{*}{(0.0)}    & 13\text{K}    & \multirow{3}{*}{(26.0)}   & 1\text{K}     & \multirow{3}{*}{(0.0)}    & 12\text{K}    & \multirow{3}{*}{(74.5)} \\
         L-10   & 233\text{K}   &                           & 32\text{M}    &                           & 22            &                           & 57\text{K}    &                           & 26\text{K}    &                           & 2\text{K}     &                           & 25\text{K}    & \\
         L-20   & 487\text{K}   &                           & 170\text{M}   &                           & 43            &                           & 121\text{K}   &                           & 55\text{K}    &                           & 4\text{K}     &                           & 53\text{K}    & \\ 
         \hline
    \end{tabular}
    
    \vspace{0.3em}
    \centering
    \small
    \setlength{\tabcolsep}{.15em}
    \begin{tabular}{|r| *{9}{>{$}r<{$}>{$}c<{$}|}}
        \hline
                    & \multicolumn{2}{>{$}c<{$}|}{q^s_1}    & \multicolumn{2}{>{$}c<{$}|}{q^s_2}    & \multicolumn{2}{>{$}c<{$}|}{q^s_3}    & \multicolumn{2}{>{$}c<{$}|}{q^s_4}    & \multicolumn{2}{>{$}c<{$}|}{q^s_5}    & \multicolumn{2}{>{$}c<{$}|}{q^s_6}    & \multicolumn{2}{>{$}c<{$}|}{q^s_{7}}  & \multicolumn{2}{>{$}c<{$}|}{q^s_{8}}  & \multicolumn{2}{>{$}c<{$}|}{q^s_{9}} \\
                    & \text{Tot}    & (\%)                  & \text{Tot}    & (\%)                  & \text{Tot}    & (\%)                  & \text{Tot}    & (\%)                  & \text{Tot}    & (\%)                  & \text{Tot}    & (\%)                  & \text{Tot}    & (\%)                  & \text{Tot}    & (\%)                  & \text{Tot}    & (\%) \\
        \hline        
        SEMINTEC    & 7             & (0.0 )                & 53            & (0.0)                 & 16            & (0.0 )                & 12            & (0.0)                 & 31            & (0.0 )                & 838\text{K}   & (55.4)                & 5\text{K}     & (0.0)                 & 5\text{K}     & (54.3)                & 13\text{K}    & (33.3) \\
        \hline
    \end{tabular}
    \setlength{\abovecaptionskip}{5pt}
    \setlength{\belowcaptionskip}{-10pt}
    \caption{Total number of answers and ratio spurious to  answers. In Table LSTW, the ratio is stable for each data set.}\label{table:spurious}
\end{table*}

To gain insight into the practical applicability of our approach, we
implemented our technique in a prototypical system. The system uses HermiT, a
widely used ontology reasoner, as a datalog engine in order to materialise the
consequences of $\dat(\K)$ and evaluate $q$. The system has been implemented
in Java, and we ran our experiments on a MacBook Pro with 4GB of RAM and an
Intel Core 2 Duo 2.4 Ghz processor. We used two ontologies in our evaluation,
details of which are given below. The ontologies, queries, and the prototype
system are all available online at
\url{http://www.cs.ox.ac.uk/isg/tools/KARMA/}.

The LSTW benchmark \cite{Tamingrolehie} consists of an OWL 2 QL version of the
LUBM ontology \cite{DBLP:journals/ws/GuoPH05}, queries ${q^l_1, \ldots,
q^l_{11}}$, and a data generator. The LSTW ontology extends the standard LUBM
ontology with several axioms of type 6 (see Table \ref{table:Xi}). To obtain
an \elho ontology, we removed inverse roles and datatypes, added 11 axioms
using 9 freshly introduced nominals, and added one axiom of type 4 (see Table
\ref{table:Xi}). These additional axioms resemble the ones in Example
\ref{example:translation}, and they were designed to test equality reasoning.
The resulting signature consists of 132 concepts, 32 roles, and 9 nominals,
and the ontology contains 180 axioms. From the 11 LSTW queries, we did not
consider queries $q^l_4$, $q^l_6$, $q^l_7$, and $q^l_{11}$ because their
result sets were empty: $q^l_4$ relies on existential quantification over
inverse roles, and the other three are empty already w.r.t.\ the original LSTW
ontology. Query $q^l_2$ is similar to query $q_2$ from Example
\ref{example:spurious}, and it was designed to produce only spurious answers
and thus stress the system. We generated data sets with $5$, $10$ and $20$
universities. For each data set, we denote with L-$i$ the knowledge base
consisting of our \elho ontology and the ABox for $i$ universities (see Table
\ref{table:size}).

SEMINTEC is an ontology about financial services developed within the SEMINTEC
project at the University of Poznan. To obtain an \elho ontology, we removed
inverse roles, role functionality axioms, and universal restrictions, added
nine axioms of type 6 (see Table \ref{table:Xi}), and added six axioms using 4
freshly introduced nominals. The resulting ontology signature consists of 60
concepts, 16 roles, and 4 nominals, and the ontology contains 173 axioms.
Queries $q^s_1$--$q^s_5$ are tree-shaped queries used in the SEMINTEC project,
and we developed queries $q^s_6$--$q^s_9$ ourselves. Query $q_6^s$ resembles
query $q^l_2$ from LSTW, and queries $q_8^s$ and $q_9^s$ were designed to
retrieve a large number of answers containing auxiliary individuals, thus
stressing condition \eqref{spur:cond1} of Definition \ref{def:isSpurious}.
Finally, the SEMINTEC ontology comes with a data set consisting of
approximately 65,000 facts concerning 18,000 individuals (see row SEM in Table
\ref{table:size}).

The practicality of our approach, we believe, is determined mainly by the
following two factors. First, the number of facts involving auxiliary
individuals introduced during the materialisation phase should not be `too
large'. Table \ref{table:size} shows the materialisation results: the first
column shows the number of individuals before and after materialisation and
the percentage of `true' auxiliary individuals, the second column shows the
number of unary facts before and after materialisation and the percentage of
facts involving a `true' auxiliary individual, and the third column does the
same for binary facts. As one can see, for each input data set, the
materialisation step introduces few `true' auxiliary individuals, and the
number of facts at most doubles. The number of unary facts involving a `true'
auxiliary individual does not change with the size of the input data set,
whereas the number of such binary facts increases by a constant factor. This
is because, in clauses of type 6, atoms $A(o_{R,A})$ do not contain a
variable, whereas atoms $R(x,o_{R,A})$ do.

Second, evaluating $q$ over $\dat(\K)$ should not produce too many spurious
answers. Table \ref{table:spurious} shows the total number of answers for each
query---that is, the number of answers obtained by evaluating the query over
$\dat(\K)$; furthermore, the table also shows what percentage of these answers
are spurious. Queries $q_2^l$, $q^l_{10}$, $q^s_{6}$, and $q^s_{8}$ retrieve a
significant percentage of spurious answers. However, only query $q_2^l$ has
proven to be challenging for our system due to the large number of retrieved
answers, with an evaluation time of about 40 minutes over the largest
knowledge base (L-20). Surprisingly, $q_1^l$ also performed rather poorly
despite a low number of spurious answers, with an evaluation time of about 20
minutes for L-20. All other queries were evaluated in at most a few seconds,
thus suggesting that queries $q_1^l$ and $q_2^l$ are problematical mainly
because HermiT does not implement query optimisation algorithms typically used
in relational databases.

\section{Conclusion}

We presented the first combined technique for answering conjunctive queries
over DL ontologies that include nominals. A preliminary evaluation suggests
the following. First, the number of materialised facts over `true' anonymous
individuals increases by a constant factor with the size of the data. Second,
query evaluation results have shown that, while some cases may be challenging,
in most cases the percentage of answers that are spurious is manageable.
Hence, our technique provides a practical CQ answering algorithm for a large
fragment of OWL 2 EL.

We anticipate several directions for our future work. First, we would like to
investigate the use of top-down query evaluation techniques, such as magic
sets \cite{DBLP:books/aw/AbiteboulHV95} or SLG resolution
\cite{Chen:1993:QEU:153850.153865}. Second, tighter integration of the
detection of spurious answers with the query evaluation algorithms should make
it possible to eagerly detect spurious answers (i.e., before the query is
fully evaluated). \citeA{Tamingrolehie} already implemented a filtering
condition as a user-defined function in a database, but it is unclear to what
extent such an implementation can be used to optimise query evaluation.
Finally, we would like to extend our approach to all of OWL 2 EL.

\section*{Acknowledgements}
This work was supported by the Royal Society; Alcatel-Lucent; the EU FP7
project OPTIQUE; and the EPSRC projects ExODA, MASI$^3$, and QueRe.

\clearpage

\bibliographystyle{aaai}
\bibliography{references}
\ifdraft{
\clearpage
\appendix
\onecolumn

\section{Additional Proofs}

\subsection{Proof of Lemma \ref{lemma:homomorphism}}

\lemmahomomorphism*

\begin{proof}
Let ${I_0, I_1, \ldots}$ be the sequence of sets used to construct $I$. We
show by induction on $n$ that each $\lpmodel[n]$ satisfies the properties.

\smallskip

\basecase Consider \lpmodel[0] and an arbitrary fact ${H \in \lpmodel[0]}$.
Each term occurring in $H$ is contained in \indnames. Moreover, $H$ is a fact
from $\Xi(\K)$ and, by definition, it is also a fact from $\dat(\K)$. Now
$\delta$ is the identity over \indnames, and \datalogmodel satisfies $H$, so
properties 1 and 2 hold. Property 3 holds vacuously since \lpmodel[0] does not
contain facts with the equality predicate.

\smallskip

\inductivestep Consider an arbitrary ${n \in \nat}$ and assume that
\lpmodel[n] satisfies properties 1--3; we show that the same holds for
\lpmodel[n+1]. Towards this goal, we consider the different clauses in
${\Xi(\K) \cup \Xi(\K)_\approx}$ that can derive fresh facts from \lpmodel[n].
We distinguish the following two cases.

First, consider an arbitrary datalog clause of the form ${\varphi \rightarrow
\psi}$ from ${\Xi(\K) \cup \Xi(\K)_\approx}$. Let $\sigma$ be an arbitrary
substitution mapping variables occurring in the clause to the terms in the
Herbrand universe of $\Xi(\K)$ such that ${\sigma(\varphi) \subseteq
\lpmodel[n]}$, so the clause derives ${\sigma(\psi) \in \lpmodel[n+1]}$. Let
$\sigma'$ be the substitution defined such that ${\sigma'(x) =
\delta(\sigma(x))}$ for each variable $x$ occurring in the clause. By the
inductive hypothesis, we have ${\sigma'(\varphi) \subseteq \datalogmodel}$.
Furthermore, by the definition of $\dat(\K)$, we have that ${\dat(\K) \cup
\dat(\K)_\approx}$ contains ${\varphi \rightarrow \psi}$. Finally, since
\datalogmodel satisfies ${\varphi \rightarrow \psi}$, we have ${\sigma'(\psi)
\in \datalogmodel}$, as required.

Second, consider arbitrary clauses from $\Xi(\K)$ of the form ${A_1(x)
\rightarrow R(x,f_{R,A}(x))}$ and ${A_1(x) \rightarrow A(f_{R,A}(x))}$, and
assume that ${A_1(w) \in \lpmodel[n]}$; hence, these clauses derive
${\setof{R(w,f_{R,A}(w)),\; A(f_{R,A}(w))} \subseteq \lpmodel[n+1]}$. By the
inductive hypothesis, we have ${A_1(\delta(w)) \in \datalogmodel}$.
Furthermore, by the definition of $\delta$, we have that ${\delta(f_{R,A}(w))
= o_{R,A}}$. Moreover, by the definition of program $\dat(\K)$, the program
contains clauses ${A_1(x) \rightarrow R(x,o_{R,A})}$ and ${A_1(x) \rightarrow
A(o_{R,A})}$. Finally, model \datalogmodel satisfies both of these clauses, so
we have ${\setof{R(\delta(w), o_{R,A})\; A(o_{R,A})} \subseteq
\datalogmodel}$, as required.
\end{proof}

\subsection{Proof of Lemma \ref{lemma:dat-embed}}

In order to prove Lemma \ref{lemma:dat-embed}, we use the properties from
Lemmas \ref{lemma:auxinds} and \ref{lemma:eq}.

\begin{lemma}\label{lemma:auxinds}
    For each term $w_2$, each role ${R \in \rolenames}$, and each concept ${A
    \in \conceptnames \cup \setof{\top}}$, if ${f_{R,A}(w_2) \in
    \restricteddomain{\lpmodel}}$, then ${\setof{R(w_2, f_{R,A}(w_2)),\;
    A(f_{R,A}(w_2))} \subseteq \lpmodel}$.
\end{lemma}

\begin{proof}
Let ${\lpmodel[0], \lpmodel[1], \ldots}$ be the sequence used to construct
$I$; we assume w.l.o.g.\ that each $\lpmodel[n+1]$ is obtained from
$\lpmodel[n]$ by applying just one clause type. We show by induction on $n$
that each $\lpmodel[n]$ satisfies the properties. For the base case, set
\lpmodel[0] clearly satisfies the property since it does not contain
functional terms. For the inductive step, assume that some \lpmodel[n]
satisfies the property, and consider an arbitrary term $w_2$, role $R$, and
concept ${A \in \conceptnames \cup \setof{\top}}$. By the construction of
$\Xi(\K)$, there are only two types of clauses that may introduce new
functional terms in $\restricteddomain{\lpmodel[n+1]}$. First, such a term may
be introduced by clauses of type 6 (see Table \ref{table:Xi}), but then the
term clearly satisfies the required property. Second, a clause of the form ${x
\approx y \rightarrow f_{R,A}(x) \approx f_{R,A}(y)}$ may be applied ${w_1
\approx w_2 \in \lpmodel[n]}$ and derive ${f_{R,A}(w_1) \approx f_{R,A}(w_2)
\in \lpmodel[n+1]}$. If ${f_{R,A}(w_2) \in \restricteddomain{\lpmodel[n]}}$,
then set \lpmodel[n+1] satisfies the required property by the induction
hypothesis. Otherwise, term $f_{R,A}(w_2)$ occurs in $\lpmodel[n+1]$ only in
equality assertions, so ${f_{R,A}(w_2) \not\in
\restricteddomain{\lpmodel[n+1]}}$, and the property holds vacuously.
\end{proof}

Let ${\datalogmodel[0], \datalogmodel[1], \ldots}$ be the sequence used to
construct the minimal Herbrand model $\datalogmodel$ of $\dat(\K)$. We assume
w.l.o.g.\ that each \datalogmodel[n+1] is obtained from \datalogmodel[n] by
applying a single clause occurring in $\dat(\K)$, apart from the clause
defining the symmetry of $\approx$ which is always applied so as to keep the
relation $\approx$ in \datalogmodel[n] symmetric. We next show that each
$\datalogmodel[n]$ satisfies the following property.

\begin{lemma}\label{lemma:eq}
    For each ${n \in \nat}$ and all terms $u_1$ and $u_2$, if ${u_1 \approx
    u_2 \in \datalogmodel[n]}$ and ${u_2 \in \aux{\datalogmodel[n]}}$, then
    ${u_1 = u_2}$.
\end{lemma}

\begin{proof}
We prove the claim by the induction on $n$. For the base case,
\datalogmodel[0] satisfies the property since \aux{\datalogmodel[0]} is empty.
For the inductive step, assume that some \datalogmodel[n] satisfies the
property; we show that the same holds for \datalogmodel[n+1]. We consider the
various clauses that may derive an equality in \datalogmodel[n+1]. The facts
derived by a clause of the form ${A(x) \rightarrow x \approx a}$ vacuously
satisfy the property since the derived fact involves terms that are not in
\aux{\datalogmodel[n+1]}. Furthermore, a fact derived in \datalogmodel[n+1] by
applying either the reflexivity or the symmetry clause satisfies the property
by the inductive hypothesis. We are left to consider the transitivity clause.
Let $u_1$, $u_2$, and $u_3$ be arbitrary terms such that ${\setof{u_1 \approx
u_2,\; u_2 \approx u_3} \subseteq \datalogmodel[n]}$, so the transitivity
clause derives ${u_1 \approx u_3 \in \datalogmodel[n+1]}$. We consider the
interesting case in which ${u_3 \in \aux{\datalogmodel[n+1]}}$, so ${u_3 \in
\aux{\datalogmodel[n]}}$. By the inductive hypothesis, we have ${u_2 = u_3}$;
but then, ${u_2 \in \aux{\datalogmodel[n]}}$, and so again, by the inductive
hypothesis, we have ${u_1 = u_2}$; finally, this implies that ${u_1 = u_3}$.
\end{proof}

\lemmadatembed*

\begin{proof}
Let ${\datalogmodel[0], \datalogmodel[1], \ldots}$ be the sequence as stated
above. We prove the claim by induction on $n$.
	
\smallskip

\basecase Consider \datalogmodel[0]. By definition, ${\Xi(\K) \cup
\Xi(\K)_\approx}$ and ${\dat(\K) \cup \dat(\K)_\approx}$ contain the same
facts, all of which only refer to the individuals in \indnames and the
predicates in ${\conceptnames \cup \rolenames \cup \setof{\top,\, \bot}}$.
Since $\delta$ is the identity over \indnames, $\aux{\datalogmodel[0]}$ is
empty and ${\datalogmodel[0] = \lpmodel[0]}$, so properties 1--5 are
satisfied.

\smallskip

\inductivestep Assume that some \datalogmodel[n] satisfies properties 1--5; we
show that the same holds for \datalogmodel[n+1]. To this end, let $w_1$ and
$w_2$ be arbitrary terms in $\restricteddomain{\lpmodel}$. We next consider
the various clauses in ${\dat(\K) \cup \dat(\K)_\approx}$ that may derive
fresh assertions in \datalogmodel[n+1].

\begin{itemize}
    \item ${A(x) \rightarrow B(x)}$. Assume that ${A(\delta(w_1)) \in
    \datalogmodel[n]}$, and so the clause derives ${B(\delta(w_1)) \in
    \datalogmodel[n+1]}$. By the inductive hypothesis, we have ${A(w_1) \in
    \lpmodel}$. Finally, since the same clause occurs in $\Xi(\K)$, we have
    ${B(w_1) \in \lpmodel}$.

    \item ${A(x) \rightarrow x \approx a}$. Assume that ${A(\delta(w_1)) \in
    \datalogmodel[n]}$, and so for ${\delta(w_2) = w_2 = a}$ the clause
    derives ${\delta(w_1) \approx \delta(w_2)}$ in $\datalogmodel[n+1]$.
    Clearly, we have ${\delta(w_2) \not\in \aux{\datalogmodel[n+1]}}$. By the
    inductive hypothesis, we have ${A(w_1) \in \lpmodel}$. Finally, since the
    same clause occurs in $\Xi(\K)$, we have ${w_1 \approx w_2 \in \lpmodel}$.

    \item ${A_1(x) \wedge A_2(x) \rightarrow A(x)}$. Assume that
    ${A_1(\delta(w_1)) \in \datalogmodel[n]}$ and ${A_2(\delta(w_1)) \in
    \datalogmodel[n]}$, and so the clause derives ${A(\delta(w_1)) \in
    \datalogmodel[n+1]}$. By the inductive hypothesis, we have
    ${\setof{A_1(w_1),\, A_2(w_1)} \subseteq \lpmodel}$. Since the same clause
    occurs in $\Xi(\K)$, we have ${A(w_1) \in \lpmodel}$.
	
    \item ${R(x,y) \wedge A_1(y) \rightarrow A(x)}$. Assume that
    $R(\delta(w_1),\delta(w_2))$ and $A_1(\delta(w_2))$ are in contained
    $\datalogmodel[n]$, and so the clause derives ${A(\delta(w_1)) \in
    \datalogmodel[n+1]}$. We have the following two cases.
	\begin{itemize}\renewcommand{\itemsep}{2pt}
		\item ${\delta(w_2) \not\in \aux{\datalogmodel[n]}}$. By the inductive
        hypothesis, we then have ${\setof{R(w_1,w_2),\; A_1(w_2)} \subseteq
        \lpmodel}$.

		\item ${\delta(w_2) \in \aux{\datalogmodel[n]}}$ and term
        $\delta(w_2)$ is an auxiliary individual of the form $o_{P,A}$. By the
        inductive hypothesis, we then have ${\setof{R(w_1,f_{P,A}(w_1)),\;
        A_1(f_{P,A}(w_1))} \subseteq \lpmodel}$.
	\end{itemize}
    In either case, since the same clause occurs in $\Xi(\K)$, we have
    ${A(w_1) \in \lpmodel}$.
	
    \item ${R(x,y) \rightarrow A(y)}$. Assume that
    ${R(\delta(w_1),\delta(w_2)) \in \datalogmodel[n]}$, so the clause derives
    ${A(\delta(w_2)) \in \datalogmodel[n+1]}$. We have the following two
    cases.
    \begin{itemize}\renewcommand{\itemsep}{2pt}
    	\item ${\delta(w_2) \not\in \aux{\datalogmodel[n]}}$. By the inductive
        hypothesis, we then have ${R(w_1,w_2) \in \lpmodel}$.

		\item ${\delta(w_2) \in \aux{\datalogmodel[n]}}$. By the inductive
        hypothesis, then there exists a term $w_1'$ such that ${R(w_1',w_2)
        \in \lpmodel}$.
    \end{itemize}
    In either case, since the same clause occurs in $\Xi(\K)$, we have
    ${A(w_2) \in \lpmodel}$.

    \item ${S(x,y) \rightarrow R(x,y)}$. Assume that
    ${S(\delta(w_1),\delta(w_2)) \in \datalogmodel[n]}$, and so the clause
    derives ${R(\delta(w_1),\delta(w_2)) \in \datalogmodel[n+1]}$. We have the
    following two cases.
	\begin{itemize}\renewcommand{\itemsep}{2pt}
		\item ${\delta(w_2) \not\in \aux{\datalogmodel[n]}}$. By the inductive
        hypothesis, we have that ${S(w_1,w_2) \in \lpmodel}$. Since the same
        clause occurs in $\Xi(\K)$, we have ${R(w_1,w_2) \in \lpmodel}$.

		\item ${\delta(w_2) \in \aux{\datalogmodel[n]}}$ and $\delta(w_2)$ is
        an auxiliary individual of the form $o_{P,A}$. By the inductive
        hypothesis, then there exists a term $w_1'$ such that
        ${\setof{S(w_1,f_{P,A}(w_1)),\; S(w_1', w_2)} \subseteq \lpmodel}$.
        Since the same clause occurs in $\Xi(\K)$, we have that
        ${\setof{R(w_1,f_{P,A}(w_1)),\; R(w_1',w_2)} \subseteq \lpmodel}$.
	\end{itemize}
	
    \item ${A_1(x) \rightarrow R(x,o_{R,A})}$. Assume that ${A_1(\delta(w_1))
    \in \datalogmodel[n]}$, so for ${\delta(w_2) = o_{R,A}}$ the clause
    derives $R(\delta(w_1),\delta(w_2))$ in $\datalogmodel[n+1]$. By the
    inductive hypothesis, we then have ${A(w_1) \in \lpmodel}$. Furthermore,
    by the definition of $\dat(\K)$, set $\Xi(\K)$ contains the clause
    ${A_1(x) \rightarrow R(x,f_{R,A}(x))}$, so we have ${R(w_1,f_{R,A}(w_1))
    \in \lpmodel}$. We have the following cases.
	\begin{itemize}\renewcommand{\itemsep}{2pt}			
		\item ${\delta(w_2) \not\in \aux{\datalogmodel[n+1]}}$. Thus, we also
        have ${\delta(w_2) \not\in \aux{\datalogmodel[n]}}$, and so there
        exists some ${c \in \indnames}$ such that ${\delta(w_2) \approx
        \delta(c) \in \datalogmodel[n]}$ and ${\delta(c) \not\in
        \aux{\datalogmodel[n]}}$. By the inductive hypothesis, we have ${w_2
        \approx c \in \lpmodel}$. Due to ${\delta(w_2) =
        \delta(f_{R,A}(w_1))}$ and the inductive hypothesis, we have ${c
        \approx f_{R,A}(w_1) \in \lpmodel}$. Since $\approx$ is a congruence
        relation and ${\setof{R(w_1, f_{R,A}(w_1)),\; c \approx
        f_{R,A}(w_1),\; c \approx w_2} \subseteq \lpmodel}$, we have
        ${R(w_1,w_2) \in \lpmodel}$, as required. By the inductive hypothesis,
        property 5 is also satisfied.
		
		\item ${\delta(w_2) \in \aux{\datalogmodel[n+1]}}$. By the definition
        of $\delta$, term $w_2$ is of the form $f_{R,A}(w_2')$, and, by the
        induction hypothesis, we have that ${f_{R,A}(w_2') \in
        \restricteddomain{\lpmodel}}$. By Lemma \ref{lemma:auxinds}, we have
        that ${R(w_2',f_{R,A}(w_2')) \in \lpmodel}$. As stated above,
        ${R(w_1,f_{R,A}(w_1)) \in \lpmodel}$, so property 3 is satisfied.
        Moreover, ${\delta(f_{R,A}(w_1)) = o_{R,A}}$, and so property 5 is
        satisfied as well.
	\end{itemize}
	
   	\item ${A_1(x) \rightarrow A(o_{R,A})}$. Assume that ${A_1(\delta(w_1))
    \in \datalogmodel[n]}$, so for ${\delta(w_2) = o_{R,A}}$ the clause
    derives ${A(\delta(w_2)) \in \datalogmodel[n+1]}$. By the definition of
    $\delta$, term $w_2$ is of the form $f_{R,A}(w_2')$. By Lemma
    \ref{lemma:auxinds} and ${w_2 \in \restricteddomain{\lpmodel}}$, we have
    ${A(w_2) \in \lpmodel}$.

    \item ${\rightarrow x \approx x}$. Assume that $\delta(w_1)$ occurs in
    \datalogmodel[n], so the clause derives ${\delta(w_1) \approx \delta(w_2)
    \in \datalogmodel[n+1]}$ with ${\delta(w_1) = \delta(w_2)}$. We consider
    the interesting case when ${\delta(w_2) \not\in
    \aux{\datalogmodel[n+1]}}$, and so ${\delta(w_2) \not\in
    \aux{\datalogmodel[n]}}$. Then, an individual ${c \in \indnames}$ exists
    such that ${\setof{\delta(w_1) \approx c,\; c \approx \delta(w_2)}
    \subseteq \datalogmodel[n]}$. By the inductive hypothesis, we have that
    ${\setof{w_1 \approx c,\; c \approx w_2} \subseteq \lpmodel}$. By the
    transitivity of $\approx$, we have ${w_1 \approx w_2 \in \lpmodel}$.
	 
    \item ${x_1 \approx x_2 \rightarrow x_2 \approx x_1}$. Assume that
    ${\delta(w_1) \approx \delta(w_2) \in \datalogmodel[n]}$, so the clause
    derives ${\delta(w_2) \approx \delta(w_1) \in \datalogmodel[n+1]}$. We
    consider the interesting case when ${\delta(w_1) \not\in
    \aux{\datalogmodel[n+1]}}$; clearly, we have ${\delta(w_1) \not\in
    \aux{\datalogmodel[n]}}$ as well. Since predicate $\approx$ is symmetric
    in \datalogmodel[n], we have ${\delta(w_2) \approx \delta(w_1) \in
    \datalogmodel[n]}$. By the inductive hypothesis, we have ${w_2 \approx w_1
    \in \lpmodel}$.
		
    \item ${x_1 \approx x_3 \wedge x_3 \approx x_2 \rightarrow x_1 \approx
    x_2}$. Assume that set \datalogmodel[n] contains ${\delta(w_1) \approx
    \delta(w_3)}$ and ${\delta(w_3) \approx \delta(w_2)}$, so the clause
    derives ${\delta(w_1) \approx \delta(w_2) \in \datalogmodel[n+1]}$. The
    only interesting case is when ${\delta(w_2) \not\in
    \aux{\datalogmodel[n+1]}}$; clearly, then ${\delta(w_2) \not\in
    \aux{\datalogmodel[n]}}$. By Lemma \ref{lemma:eq}, then ${\delta(w_3)
    \not\in \aux{\datalogmodel[n]}}$. Finally, by the inductive hypothesis,
    then ${\setof{w_1 \approx w_3,\; w_3 \approx w_2} \subseteq \lpmodel}$,
    which implies ${w_1 \approx w_2 \in \lpmodel}$.

    \item ${A(x) \wedge x \approx y \rightarrow A(y)}$. Assume that facts
    $A(\delta(w_1))$ and ${\delta(w_1) \approx \delta(w_2)}$ are contained in
    \datalogmodel[n], so the clause derives ${A(\delta(w_2)) \in
    \datalogmodel[n+1]}$. By the inductive hypothesis, we have ${A(w_1) \in
    \lpmodel}$. We consider the following two cases.
	\begin{itemize}\renewcommand{\itemsep}{2pt}
        \item ${\delta(w_2) \not\in \aux{\datalogmodel[n]}}$. By the inductive
        hypothesis, we have ${w_1 \approx w_2 \in \lpmodel}$, and so ${A(w_2)
        \in \lpmodel}$.

        \item ${\delta(w_2) \in \aux{\datalogmodel[n]}}$. By Lemma
        \ref{lemma:eq}, then ${\delta(w_1) = \delta(w_2)}$, so
        ${A(\delta(w_2)) \in \datalogmodel[n]}$. Finally, by the inductive
        hypothesis, we then have ${A(w_2) \in \lpmodel}$.
    \end{itemize}

    \item ${R(x,y) \wedge x\approx z \rightarrow R(z,y)}$. Assume that set
    \datalogmodel[n] contains $R(\delta(w_1),\delta(w_2))$ and ${\delta(w_1)
    \approx \delta(w_3)}$, so the clause derives ${R(\delta(w_3),\delta(w_2))
    \in \datalogmodel[n+1]}$. We consider the following two cases.
 	\begin{itemize}\renewcommand{\itemsep}{2pt}
        \item ${\delta(w_2) \not\in \aux{\datalogmodel[n]}}$. By the inductive
        hypothesis, we have ${R(w_1,w_2) \in \lpmodel}$. We distinguish two
        additional cases. First, assume that ${\delta(w_3) \in
        \aux{\datalogmodel[n]}}$. By Lemma \ref{lemma:eq}, we have
        ${\delta(w_1) = \delta(w_3)}$, and so ${R(\delta(w_3),\delta(w_2)) \in
        \datalogmodel[n]}$. By the inductive hypothesis, then ${R(w_3,w_2) \in
        \lpmodel}$. Second, assume that ${\delta(w_3) \not\in
        \aux{\datalogmodel[n]}}$. By the inductive hypothesis, we have ${w_1
        \approx w_3 \in \lpmodel}$, and so we have ${R(w_3,w_2) \in \lpmodel}$
        as well.

        \item ${\delta(w_2) \in \aux{\datalogmodel[n]}}$ and ${\delta(w_2) =
        o_{P,A}}$. By the inductive hypothesis, some $w_1'$ exists s.t.\
        ${\setof{R(w_1,f_{P,A}(w_1)),\; R(w_1',w_2)} \subseteq \lpmodel}$. We
        distinguish two additional cases. First, assume that ${\delta(w_3) \in
        \aux{\datalogmodel[n]}}$. By Lemma \ref{lemma:eq}, we have
        ${\delta(w_1) = \delta(w_3)}$, which further implies
        ${R(\delta(w_3),\delta(w_2)) \in \datalogmodel[n]}$. By the inductive
        hypothesis, then we have ${\setof{R(w_3,f_{P,A}(w_3)),\; R(w_1',w_2)}
        \subseteq \lpmodel}$. Second, assume that ${\delta(w_3) \not\in
        \aux{\datalogmodel[n]}}$. By the inductive hypothesis, we have ${w_1
        \approx w_3 \in\lpmodel}$. By the functional reflexivity clauses, then
        ${f_{R,B}(w_1) \approx f_{R,B}(w_3) \in \lpmodel}$, which again
        implies ${\setof{R(w_3,f_{R,B}(w_3)),\; R(w_1', w_2)} \subseteq
        \lpmodel}$.
 	\end{itemize}

    \item ${R(x,y) \wedge y \approx z \rightarrow R(x,z)}$. Assume that set
    \datalogmodel[n] contains $R(\delta(w_1),\delta(w_2))$ and ${\delta(w_2)
    \approx \delta(w_3)}$, so the clause derives ${R(\delta(w_1),\delta(w_3))
    \in \datalogmodel[n+1]}$. We consider the following two cases.
	\begin{itemize}\renewcommand{\itemsep}{2pt}
        \item ${\delta(w_2) \not\in \aux{\datalogmodel[n]}}$. By Lemma
        \ref{lemma:eq}, then ${\delta(w_3) \not \in\aux{\datalogmodel[n]}}$,
        and so ${\delta(w_3) \not\in \aux{\datalogmodel[n+1]}}$ as well. By
        the inductive hypotheses, then $R(w_1,w_2)$ and ${w_2 \approx w_3}$
        are in $\lpmodel$, so ${R(w_1,w_3) \in \lpmodel}$ as well.

        \item ${\delta(w_2) \in \aux{\datalogmodel[n]}}$ and $\delta(w_2)$ is
        of the form $o_{P,A}$. By Lemma \ref{lemma:eq}, then ${\delta(w_3) =
        \delta(w_2)}$, which implies ${R(\delta(w_1),\delta(w_3)) \in
        \datalogmodel[n]}$. Finally, by the inductive hypothesis, then there
        exists a term $w_1'$ such that ${\setof{R(w_1,f_{R,B}(w_1)),\;
        R(w_1',w_3)} \subseteq \lpmodel}$. \qedhere
    \end{itemize}
\end{itemize}
\end{proof}

\subsection{Proof of Proposition \ref{prop:sat}}

\propsat*

\begin{proof}
From Lemmas \ref{lemma:homomorphism} and \ref{lemma:dat-embed}, we have
${\bot(w) \in \lpmodel}$ if and only if ${\bot(\delta(w)) \in \datalogmodel}$.
Thus, \K is unsatisfiable if and only if individual $u$ exists such that
${\dat(\K) \models \bot(u)}$. Furthermore, to check the latter, we can compute
\datalogmodel and check whether an individual $u$ exists such that ${\bot(u)
\not\in \datalogmodel}$. Since the number of variables occurring in each
datalog clause is bounded by a constant, the computation of \datalogmodel can
be implemented in polynomial time in the size of
\K~\cite{DBLP:journals/csur/DantsinEGV01}.
\end{proof}

\subsection{Proof of Theorem \ref{th:main}}

We first show that the minimal Herbrand model \lpmodel of $\Xi(\K)$ resembles
a forest structure. Let ${\lpmodel[0], \lpmodel[1], \ldots}$ be the sets used
to generate \lpmodel; for simplicity, in the rest of this section we assume
w.l.o.g.\ that the clauses are applied in a way so that relation $\approx$ is
symmetric in each \lpmodel[n]. Furthermore, for each term $w$, we define the
\emph{size of $w$} as follows.
\begin{displaymath}
    \card{w} = \begin{cases}
                    0               & \text{if } w\in\indnames, \\
                    1 + \card{w'}   & \text{if } w \text{ is of the form } f_{T,A}(w'). \\
               \end{cases}
\end{displaymath}
Finally, we define the \emph{depth of $w$ in \lpmodel} as follows.
\begin{displaymath}
    \depth{w,\lpmodel} = \begin{cases}
                            0                       & \text{if } w\not\in\aux{\lpmodel}, \\
                            1 + \depth{w',\lpmodel} & \text{if }w\in\aux{\lpmodel} \text{ and } w=f_{T,A}(w'). \\
                         \end{cases}
\end{displaymath}

\begin{lemma}\label{lemma:treemodel}
    Interpretation \lpmodel satisfies the following three properties for all
    terms $w_1$, $w_1'$, $w_2$, and $w_2'$, all roles $R$, $S$, and $T$, and
    each concept ${A \in \conceptnames \cup \setof{\top}}$.
    \begin{enumerate}[P1.]\renewcommand{\itemsep}{3pt}
        \item ${R(w_1',f_{T,A}(w_1)) \in \lpmodel}$, ${S(w_2',f_{T,A}(w_2))
        \in \lpmodel}$, \\ ${f_{T,A}(w_1) \approx f_{T,A}(w_2) \in \lpmodel}$,
        and ${f_{T,A}(w_2) \in \aux{\lpmodel}}$ imply ${w_1' \approx w_2' \in
        \lpmodel}$.

        \item ${w_1 \approx w_2 \in \lpmodel}$ implies ${\depth{w_1,\lpmodel}
        = \depth{w_2,\lpmodel}}$.

        \item ${R(w_1',f_{T,A}(w_1)) \in \lpmodel}$ and ${f_{T,A}(w_1) \in
        \aux{\lpmodel}}$ imply that ${\depth{f_{T,A}(w_1),\lpmodel} = 1 +
        \depth{w_1',\lpmodel}}$.
    \end{enumerate}
\end{lemma}

\begin{proof}
To prove properties P1--P3, we first show by induction on $n$ that each
\lpmodel[n] satisfies the following two auxiliary properties for all terms
$w'$, $w$, $w_1$, and $w_2$, all roles $R$, $T$, and $T'$, and all concepts
$A$ and $A'$ in ${\conceptnames \cup \setof{\top}}$.
\begin{enumerate}[{A}1.]\renewcommand{\itemsep}{3pt}
    \item ${f_{T',A'}(w_2) \in \aux{\lpmodel[n]}}$ and ${f_{T,A}(w_1) \approx
    f_{T',A'}(w_2) \in \lpmodel[n]}$ imply that $T = T'$, $A = A'$, and ${w_1
    \approx w_2 \in \lpmodel}$.

    \item ${f_{T,A}(w) \in \aux{\lpmodel[n]}}$ and ${R(w',f_{T,A}(w)) \in
    \lpmodel[n]}$ imply that a term $w''$ exists such that \lpmodel contains
    ${T(w'',f_{T,A}(w''))}$, ${w' \approx w''}$, and ${f_{T,A}(w) \approx
    f_{T,A}(w'')}$.
\end{enumerate}

\smallskip

\basecase By definition, \lpmodel[0] does not contain functional terms, so
properties A1 and A2 are vacuously true.

\smallskip

\inductivestep Assume that \lpmodel[n] satisfies properties A1 and A2; we show
that the same holds for \lpmodel[n+1] by considering in turn the various
clauses that may introduce fresh assertions into \lpmodel[n+1]. We consider
only the interesting cases in where an equality or a binary assertion is
derived, since all other clauses trivially preserve A1 and A2. Let $w'$, $w$,
$w_1$, $w_1'$, and $w_2$ be arbitrary terms, let $R$, $T$, and $T'$ be
arbitrary roles, and let $A$ and $A'$ be arbitrary concepts in ${\conceptnames
\cup \setof{\top}}$.

\begin{itemize}
    \item ${A_1(x) \rightarrow x \approx a}$. Assume that ${A_1(w_1) \in
    \lpmodel[n]}$, so the clause derives ${w_1 \approx a \in \lpmodel[n+1]}$.
    Since ${a \not\in \aux{\lpmodel[n+1]}}$, properties A1 and A2 are
    preserved.

    \item ${\rightarrow x \approx x}$. Assume that $f_{T,A}(w_1)$ occurs in
    \lpmodel[n], so the clause derives ${f_{T,A}(w_1) \approx f_{T,A}(w_1) \in
    \lpmodel[n+1]}$; the interesting case is when ${f_{T,A}(w_1) \in
    \aux{\lpmodel[n+1]}}$. Since $f_{T,A}(w_1)$ occurs in \lpmodel[n], then
    $w_1$ occurs in the Herbrand universe of $\Xi(\K)$. By reflexivity, then
    ${w_1 \approx w_1 \in \lpmodel}$, as required for A1. Furthermore, this
    derivation clearly preserves A2.

    \item ${x \approx y \rightarrow f_{T,A}(x) \approx f_{T,A}(y)}$. Assume
    ${w_1 \approx w_2 \in \lpmodel[n]}$, so the clause derives ${f_{T,A}(w_1)
    \approx f_{T,A}(w_2) \in \lpmodel[n+1]}$; the interesting case is when
    ${f_{T,A}(w_2) \in \aux{\lpmodel[n+1]}}$. By assumption, ${w_1 \approx w_2
    \in \lpmodel[n]}$, and so ${w_1 \approx w_2 \in \lpmodel}$, as required
    for property A1. Furthermore, this derivation clearly preserves A2.

    \item ${x_1 \approx x_2 \rightarrow x_2 \approx x_1}$. Assume that
    ${f_{T',A'}(w_2) \approx f_{T,A}(w_1) \in \lpmodel[n]}$, so the clause
    derives ${f_{T,A}(w_1) \approx f_{T',A'}(w_2) \in \lpmodel[n+1]}$; the
    interesting case is when ${f_{T',A'}(w_2) \in \aux{\lpmodel[n+1]}}$, which
    clearly implies ${f_{T',A'}(w_2) \in \aux{\lpmodel[n]}}$. Since relation
    $\approx$ is symmetric in $\lpmodel[n]$, we have ${f_{T,A}(w_1) \approx
    f_{T',A'}(w_2) \in \lpmodel[n]}$; but then, by the inductive hypothesis,
    we have ${T = T'}$, ${A = A'}$, and ${w_1 \approx w_2 \in \lpmodel}$, as
    required for property A1. Furthermore, this derivation clearly preserves
    A2.

    \item ${x_1 \approx x_2 \wedge x_2 \approx x_3 \rightarrow x_1 \approx
    x_3}$. Assume that \lpmodel[n] contains ${f_{T,A}(w_1) \approx
    f_{T',A'}(w_1')}$ and ${f_{T',A'}(w_1') \approx f_{T'',A''}(w_2)}$, so the
    clause derives ${f_{T,A}(w_1) \approx f_{T'',A''}(w_2) \in
    \lpmodel[n+1]}$; the interesting case is when ${f_{T'',A''}(w_2) \in
    \aux{\lpmodel[n+1]}}$, which clearly implies ${f_{T'',A''}(w_2) \in
    \aux{\lpmodel[n]}}$. Clearly, we then also have ${f_{T',A'}(w_1') \in
    \aux{\lpmodel[n]}}$. By the inductive hypothesis, we have ${T = T' =
    T''}$, ${A = A' = A''}$, and ${\setof{w_1 \approx w_1',\; w_1' \approx
    w_2} \subseteq \lpmodel}$. Thus, we have ${w_1 \approx w_2 \in \lpmodel}$,
    as required for property A1. Furthermore, this derivation clearly
    preserves A2.

    \item ${A_1(x) \rightarrow T(x,f_{T,A}(x))}$. Assume that ${A_1(w') \in
    \lpmodel[n]}$, so the clause derives ${T(w',f_{T,A}(w')) \in
    \lpmodel[n+1]}$; the interesting case is when ${f_{T,A}(w') \in
    \aux{\lpmodel[n+1]}}$ and ${w' = w}$. Then, for ${w'' = w = w'}$, we have
    \begin{displaymath}
        \setof{T(w'',f_{T,A}(w'')),\; w'' \approx w'',\; f_{T,A}(w'') \approx f_{T,A}(w'')} \subseteq \lpmodel,
    \end{displaymath}
    as required for property A2. Furthermore, this derivation clearly
    preserves A1.

    \item ${P(x,y) \rightarrow R(x,y)}$. Assume that ${P(w',f_{T,A}(w)) \in
    \lpmodel[n]}$, so the clause derives ${R(w',f_{T,A}(w)) \in
    \lpmodel[n+1]}$; the interesting case is when ${f_{T,A}(w) \in
    \aux{\lpmodel[n+1]}}$, which implies ${f_{T,A}(w) \in \aux{\lpmodel[n]}}$.
    By the inductive hypothesis, then a term $w''$ exists such that
    ${\setof{T(w'',f_{T,A}(w'')),\; w'\approx w'',\; f_{T,A}(w) \approx
    f_{T,A}(w'')} \subseteq \lpmodel}$, as required for property A2.
    Furthermore, this derivation clearly preserves A1.

    \item ${R(x,y) \wedge x \approx z \rightarrow R(z,y)}$. Assume that
    ${\setof{R(w_1',f_{T,A}(w)),\; w_1' \approx w'} \subseteq \lpmodel[n]}$,
    so the clause derives ${R(w',f_{T,A}(w)) \in \lpmodel[n+1]}$; the
    interesting case is when ${f_{T,A}(w) \in \aux{\lpmodel[n+1]}}$, which
    implies ${f_{T,A}(w) \in \aux{\lpmodel[n]}}$. By the inductive hypothesis,
    a term $w''$ exists such that ${\setof{T(w'',f_{T,A}(w'')),\; w_1' \approx
    w'',\; f_{T,A}(w) \approx f_{T,A}(w'')} \subseteq \lpmodel}$. By the
    transitivity of $\approx$, we have ${w' \approx w'' \in \lpmodel}$, as
    required for property A2. Furthermore, this derivation clearly preserves
    A1.

    \item ${R(x,y) \wedge y \approx z \rightarrow R(x,z)}$. Assume that
    ${\setof{R(w',f_{T,A}(w_1)),\; f_{T,A}(w_1) \approx f_{T,A}(w)} \subseteq
    \lpmodel[n]}$, and so the clause derives the fact ${R(w',f_{T,A}(w)) \in
    \lpmodel[n+1]}$; the interesting case is when ${f_{T,A}(w) \in
    \aux{\lpmodel[n+1]}}$, which implies ${f_{T,A}(w) \in \aux{\lpmodel[n]}}$.
    Then, clearly ${f_{T,A}(w_1) \in \aux{\lpmodel[n]}}$. By the inductive
    hypothesis, a term $w''$ exists such that
    \begin{displaymath}
        \setof{T(w'',f_{T,A}(w'')),\; w' \approx w'',\; f_{T,A}(w_1) \approx f_{T,A}(w'')} \subseteq \lpmodel.
    \end{displaymath}
    By the transitivity of $\approx$, then ${f_{T,A}(w) \approx f_{T,A}(w'')
    \in \lpmodel}$, as required for property A2. Furthermore, this derivation
    clearly preserves A1.
\end{itemize}
We are now ready to show properties P1--P3.

\medskip

\textsc{Property P1}.
Let $w_1'$, $w_1$, $w_2'$, $w_2$ be arbitrary terms, let $R$, $S$, and $T$ be
arbitrary roles, and let $A$ be an arbitrary concept in ${\conceptnames \cup
\setof{\top}}$. Assume that ${\setof{R(w_1',f_{T,A}(w_1)),\;
S(w_2',f_{T,A}(w_2)),\; f_{T,A}(w_1) \approx f_{T,A}(w_2)} \subseteq
\lpmodel}$ and ${f_{T,A}(w_2) \in \aux{\lpmodel}}$. By applying property A2 to
$R(w_1',f_{T,A}(w_1))$ and $S(w_2',f_{T,A}(w_2))$, we have that two terms
$w_1''$ and $w_2''$ exist such that
\begin{align*}
    \setof{T(w_1'',f_{T,A}(w_1'')),\, w_1'\approx w_1'',  f_{T,A}(w_1)\approx f_{T,A}(w_1'')}\subseteq\lpmodel, & \text{ and}\\
    \setof{T(w_2'',f_{T,A}(w_2'')),\, w_2'\approx w_2'',  f_{T,A}(w_2)\approx f_{T,A}(w_2'')}\subseteq\lpmodel. &
\end{align*}
By the transitivity of $\approx$, we have that ${f_{T,A}(w_1'') \approx
f_{T,A}(w_2'') \in \lpmodel}$, and so by Property A1, we conclude that ${w_1''
\approx w_2'' \in \lpmodel}$. Finally, since ${\setof{w'_1 \approx w_1'',\;
w_1'' \approx w_2'',\; w_2'' \approx w_2'} \subseteq \lpmodel}$, by the
transitivity of $\approx$, we get ${w'_1 \approx w'_2 \in \lpmodel}$, as
required.

\medskip

\textsc{Property P2}.
We show by induction on ${n \in \nat}$ that, for all terms $w_1$ and $w_2$
such that ${\card{w_1} \leq \card{w_2} \leq n}$, if ${w_1 \approx w_2 \in
\lpmodel}$, then ${\depth{w_1,\lpmodel} = \depth{w_2,\lpmodel}}$.

\smallskip

\basecase Let $w_1$ and $w_2$ be arbitrary terms such that ${\card{w_1} =
\card{w_2} = 0}$ and ${w_1 \approx w_2\in \lpmodel}$. By the definition of
$\card{\cdot}$, then ${\setof{w_1, w_2} \subseteq \indnames}$, so
${\depth{w_1,\lpmodel} = \depth{w_2,\lpmodel} = 0}$.

\smallskip

\inductivestep Consider an arbitrary ${n \in \nat}$ and assume that the
required property holds for all terms $w'_1$ and $w'_2$ such that
${\card{w'_1} \leq \card{w'_2} \leq n}$; we show that the same holds for
arbitrary terms $w_1$ and $w_2$ such that ${\card{w_1} \leq \card{w_2} \leq
n+1}$. We consider the interesting case when ${w_1 \approx w_2 \in \lpmodel}$,
for which we consider two cases. First, if ${w_2 \not\in \aux{\lpmodel}}$,
then ${\depth{w_1,\lpmodel} = \depth{w_2,\lpmodel} = 0}$. Second, if ${w_2 \in
\aux{\lpmodel}}$, then by property A1 there exist two terms $w_1'$ and $w_2'$,
a role $T$, and a concept ${A \in \conceptnames \cup \setof{\top}}$ such that
$w_1$ is of the form $f_{T,A}(w'_1)$, term $w_2$ is of the form
$f_{T,A}(w'_2)$, and ${w'_1 \approx w'_2 \in \lpmodel}$. By the inductive
hypothesis, then ${\depth{w'_1,\lpmodel} = \depth{w'_2,\lpmodel}}$. Finally,
by definition, we have ${\depth{w_1,\lpmodel} = \depth{w_2,\lpmodel} = 1+
\depth{w'_2,\lpmodel}}$, as required.

\medskip

\textsc{Property P3}.
Let $w_1'$ and $w_1$ be arbitrary terms, let $R$ and $T$ be arbitrary roles,
let ${A \in \conceptnames \cup \setof{\top}}$ be an arbitrary concept, and
assume that ${R(w_1',f_{T,A}(w_1)) \in \lpmodel}$ and ${f_{T,A}(w_1) \in
\aux{\lpmodel}}$. By property A2, then there exists a term $w_1''$ such that
${\setof{T(w_1'',f_{T,A}(w_1'')),\; w_1' \approx w_1'',\; f_{T,A}(w_1) \approx
f_{T,A}(w_1'')} \subseteq \lpmodel}$. By the definition of $\depth{\cdot}$,
then ${\depth{f_{T,A}(w_1''),\lpmodel} = 1 + \depth{w_1'',\lpmodel}}$.
Furthermore, by property P2, then ${\depth{f_{T,A}(w_1),\lpmodel} =
\depth{f_{T,A}(w_1''),\lpmodel}}$ and ${\depth{w_1',\lpmodel} =
\depth{w_1'',\lpmodel}}$. Finally, these observations imply that
${\depth{f_{T,A}(w_1),\lpmodel} = 1 + \depth{w_1',\lpmodel}}$, as required.
\end{proof}

We now have all the ingredients required to prove Theorem \ref{th:main}. We
start by showing completeness.

\begin{lemma}[Completeness]
    Let ${\K = \tuple{\T,\A}}$ be a satisfiable \elho KB, let ${q = \exists
    \vec{y}. \psi(\vec{x},\vec{y})}$ be a CQ, and let ${\pi : \vec{x} \mapsto
    \indnames}$ be a candidate answer for $q$. Then, ${\Xi(\K) \models
    \pi(q)}$ implies that a substitution $\tau$ exists such ${\dom{\tau} =
    \vec{x} \cup \vec{y}}$, ${\sproj{\tau}{\vec{x}} = \pi}$, ${\dat(\K)
    \models \tau(q)}$, and ${\spurious{q}{\dat(\K)}{\tau} = \false}$.
\end{lemma}

\begin{proof}
Let \lpmodel and \datalogmodel be the minimal Herbrand models of $\Xi(\K)$ and
$\dat(\K)$, respectively. Since ${\Xi(\K) \models \pi(q)}$, a substitution
$\pi^*$ exists such that ${\dom{\pi^*} = \vec{x} \cup \vec{y}}$,
${\sproj{\pi^*}{\vec{x}} = \pi}$, and ${\pi^*(q) \subseteq \lpmodel}$. Let
$\delta$ be the mapping from \lpmodel to \datalogmodel defined in the section
about the datalog rewriting of \K. We define $\tau$ as the substitution such
that, for each term ${t \in \terms{q}}$, we have ${\tau(t) :=
\delta(\pi^*(t))}$. Finally, let $\sim$ be the relation for $\tau$, $q$, and
$\dat(\K)$ as specified in Definition \ref{def:sim}. Since $\delta$ is a
homomorphism from \lpmodel to \datalogmodel by Lemma \ref{lemma:homomorphism},
we have ${\datalogmodel \models \tau(q)}$. We next prove
${\spurious{q}{\tau}{\dat(\K)} = \false}$ by showing that all conditions of
Definition \ref{def:isSpurious} are satisfied.

\smallskip

(Condition \ref{spur:cond1}) By the definition of $\tau$, for each ${x \in
\vec{x}}$, we have ${\tau(x) \in \indnames}$.

\smallskip

(Condition \ref{spur:cond2}) We prove that, for each ${s \sim t}$, we have
${\tau(s) \approx \tau(t) \in \datalogmodel}$ and ${\pi^*(s) \approx \pi^*(t)
\in \lpmodel}$. We proceed by induction on the number of steps required to
derive ${s\sim t}$. For the base case, the empty relation $\sim$ clearly
satisfies the two properties. For the inductive step, consider an arbitrary
relation $\sim$ obtained in $n$ steps that satisfies these constraints; we
show that the same holds for all constraints derivable from $\sim$. Since
relation $\approx$ in both \datalogmodel and \lpmodel is reflexive, symmetric,
and transitive, the derivation of ${s \sim t}$ due to reflexivity, symmetry,
or transitivity clearly preserves the required properties; thus, we focus on
the $(\mathsf{fork})$ rule. Let $s'$, $s$, $t'$, and $t$ be arbitrary terms in
\terms{q}, and let $R$ and $P$ be arbitrary roles such that ${s' \sim t'}$ is
obtained in $n$ steps, atoms $R(s,s')$ and $P(t,t')$ occur in $q$, and
${\tau(s') \in \aux{\dat(\K)}}$. By the inductive hypothesis, we have
${\tau(s') \approx \tau(t') \in \datalogmodel}$ and ${\pi^*(s') \approx
\pi^*(t') \in \lpmodel}$. Since \datalogmodel is the minimal Herbrand model of
$\dat(\K)$, we have ${\tau(t') \in \aux{\datalogmodel}}$, so no individual ${c
\in \indnames}$ exists such that ${\tau(t') \approx c \in \datalogmodel}$. By
Lemmas \ref{lemma:homomorphism} and \ref{lemma:dat-embed}, ${\tau(t') \not\in
\aux{\datalogmodel}}$ if and only if ${\pi^*(t') \not\in \aux{\lpmodel}}$;
hence, ${\pi^*(t') \in \aux{\lpmodel}}$. Since
${\setof{R(\pi^*(s),\pi^*(s')),\; P(\pi^*(t),\pi^*(t')),\; \pi^*(s') \approx
\pi^*(t')} \subseteq \lpmodel}$ and ${\pi^*(t') \in \aux{\lpmodel}}$, by
property P1 of Lemma \ref{lemma:treemodel} we have ${\pi^*(s) \approx \pi^*(t)
\in \lpmodel}$. Finally, since $\delta$ is a homomorphism (see Lemma
\ref{lemma:homomorphism}), by the construction of $\tau$ we have ${\tau(s)
\approx \tau(t) \in \datalogmodel}$, as required.

(Condition \ref{spur:cond3}) To show that $q$ is aux-acyclic w.r.t.\ $\tau$
and $\dat(\K)$, we assume the opposite; hence, there exists a sequence of
vertices ${v_0, \ldots, v_m \in V_\mathsf{aux}}$ such that $m > 0$, for each
${0 \leq i < m}$ we have ${\tuple{v_i,v_{i+1}} \in E_\mathsf{aux}}$, and ${v_m
= v_0}$. Consider an arbitrary ${i \leq m}$ and the corresponding edge
${\tuple{v_i,v_{i+1}} \in E_\mathsf{aux}}$. By the definition of
$E_\mathsf{aux}$, an atom $R_i(s_i,s_{i+1})$ exists in $q$ such that
${\gamma(s_i) = v_i}$ and ${\gamma(s_{i+1}) = v_{i+1}}$; hence, we have ${s_i
\sim v_i}$ and ${s_{i+1} \sim v_{i+1}}$. Since $\tau$ satisfies all the
constraints in $\sim$, by the definition of $G_\mathsf{aux}$ we have that
${\setof{\tau(s_i),\; \tau(s_{i+1})} \subseteq \aux{\dat(\K)}}$. By Lemmas
\ref{lemma:homomorphism} and \ref{lemma:dat-embed}, then
${\setof{\pi^*(s_i),\; \pi^*(s_{i+1})}\subseteq \aux{\lpmodel}}$ as well. In
addition, since $R_i(s_i,s_{i+1})$ is an atom in $q$, we have
${R_i(\pi^*(s_i),\pi^*(s_{i+1})) \in \lpmodel}$. Also, since ${s_i \sim v_i}$,
${s_{i+1} \sim v_{i+1}}$, and substitution $\pi^*$ satisfies the constraints
in $\sim$, we have ${\setof{\pi^*(s_i) \approx \pi^*(v_i),\; \pi^*(s_{i+1})
\approx \pi^*(v_{i+1})} \subseteq \lpmodel}$. Then, by property P2 of Lemma
\ref{lemma:treemodel}, we have
\begin{align*}
    \depth{\pi^*(s_{i}),\lpmodel} =     &\;\depth{\pi^*(v_i),\lpmodel} \qquad \text{and}\\
    \depth{\pi^*(s_{i+1}),\lpmodel} =   &\;\depth{\pi^*(v_{i+1}),\lpmodel}.
\end{align*}
Finally, since ${R(\pi^*(s_i),\pi^*(s_{i+1})) \in \lpmodel}$, by property P3
of Lemma \ref{lemma:treemodel} we have ${\depth{\pi^*(v_{i+1}),\lpmodel} = 1 +
\depth{\pi^*(v_i),\lpmodel}}$. But then, since $v_m = v_0$, we also have
${\depth{\pi^*(v_{m}),\lpmodel} = \depth{\pi^*(v_0),\lpmodel}}$, which is a
contradiction.
\end{proof}

We are left to prove the soundness of our approach. Let $\tau$ be an arbitrary
substitution for $q$ w.r.t.\ $\dat(\K)$ such that ${\datalogmodel \models
\tau(q)}$ and ${\spurious{q}{\tau}{\dat(\K)} = \false}$. Furthermore, let
$G_\mathsf{aux}$ be the graph as specified in Definition \ref{def:auxgraph}.
We next show that ${\lpmodel \models \sproj{\tau}{\vec{x}}(q)}$. In order to
do so, we first define the graph $G_q$ of the query $q$.

\begin{definition}
    Let $\gamma$ and $V_\mathsf{aux}$ be as in Definition \ref{def:auxgraph}.
    The \emph{query graph} ${G_q = \tuple{V_q,E_q}}$ is the directed graph
    defined as follows.
    \begin{itemize}
        \item $V_q$ is the smallest set containing $\gamma(t)$ for each ${t
        \in \terms{q}}$.
    
        \item $E_q$ is the smallest set containing
        $\tuple{\gamma(s),\gamma(t)}$ for all terms ${\setof{s,t} \subseteq
        \terms{q}}$ such that query $q$ contains $R(s,t)$ for some $R$.
    \end{itemize}
    Vertex ${v \in V_q}$ is a \emph{root} if ${v \not\in V_\mathsf{aux}}$ or,
    for each vertex ${v' \in V_q}$, we have ${\tuple{v',v} \not\in E_q}$.
\end{definition}

Clearly, by the definition, $G_\mathsf{aux}$ is a subgraph $G_q$. We prove the
soundness claim in three steps. First, we show that the graph $G_q$ is a
forest. Second, we define by structural induction on the forest $G_q$ a
substitution $\pi$ for $q$ w.r.t.\ $\Xi(\K)$ such that ${\sproj{\tau}{\vec{x}}
= \sproj{\pi}{\vec{x}}}$. Third, we prove that ${\lpmodel \models \pi(q)}$
holds.

\begin{lemma}\label{lemma:forest}
    If ${\spurious{q}{\tau}{\dat(\K)} = \false}$, then $G_q$ is a forest.
\end{lemma}

\begin{proof}
Due to ${\spurious{q}{\tau}{\dat(\K)} = \false}$, we have that
$G_\mathsf{aux}$ is a direct acyclic graph. Consider an arbitrary vertex ${v
\in V_\mathsf{aux}}$ and arbitrary vertices ${v_1, v_2 \in V_q}$ such that
${\setof{\tuple{v_1,v},\, \tuple{v_2,v}} \subseteq E_q}$; we next show that
${v_1 = v_2}$. By the definition of $G_q$, we have that ${\setof{s, s', t, t'}
\subseteq \terms{q}}$, and that roles $R$ and $P$ exist such that all of the
following conditions are satisfied:
\begin{itemize}
	\item atoms $R(s,s')$ and $P(t,t')$ are in $q$; 

    \item ${\gamma(s') = v = \gamma(t')}$, ${\gamma(s) = v_1}$, and
    ${\gamma(t) = v_2}$; and,

	\item ${\setof{\tau(s'),\tau(t')} \subseteq \aux{\datalogmodel}}$.
\end{itemize}
Due to the $(\mathsf{fork})$ rule, we have ${s \sim t}$. By the definition of
$\gamma$, we have ${\gamma(s) = \gamma(t)}$, which implies ${v_1 = v_2}$, as
required.
\end{proof}

By structural induction on the forest-shaped graph $G_q$, we next define the
substitution $\pi$ as follows; we will later show that ${\Xi(\K) \models
\pi(q)}$.
\begin{itemize}
    \item For the base case, let $v$ be an arbitrary root of $G_q$. For each
    term ${t \in \terms{q}}$ such that ${\gamma(t) = v}$, we define $\pi(t)$
    as an arbitrary term ${w \in \restricteddomain{\lpmodel}}$ such that
    ${\delta(w) = \tau(t)}$.

    \item For the inductive step, let $v$ be an arbitrary vertex of $G_q$ such
    that ${v \in V_\mathsf{aux}}$, term $\tau(v)$ is of the form $o_{R,A}$,
    the value of $\pi(v)$ is undefined, $v'$ is the unique vertex of $G_q$
    such that ${\tuple{v',v} \in E_q}$, and $\pi(v')$ has already been
    defined. For each term ${t \in \terms{q}}$ such that ${\gamma(t) = v}$, we
    define ${\pi(t) := f_{R,A}(\pi(v'))}$.
\end{itemize}

\begin{lemma}\label{lemma:pi}
    Substitution $\pi$ satisfies the two following properties for each term
    ${v \in V_q}$ and all terms ${s,t \in \terms{q}}$ such that ${\gamma(s) =
    v = \gamma(t)}$:
    \begin{enumerate}[{M}1.]
    	\item ${\delta(\pi(s)) = \tau(s)}$, and

    	\item ${\pi(s) \approx \pi(t) \in \lpmodel}$.
    \end{enumerate}
\end{lemma}

\begin{proof}
We prove properties M1 and M2 by the structural induction on the forest $G_q$.

\smallskip

\basecase Let $v$ be an arbitrary root of $G_q$, and let ${s, t \in
\terms{q}}$ be arbitrary terms such that ${\gamma(s) = v = \gamma(t)}$.
Property M1 follows from the fact that ${\pi(s) \in \setof{w \in
\restricteddomain{\lpmodel} \mid \delta(w) = \tau(s)}}$. We next prove
property M2. By the definition of $\gamma$, we have that ${s \sim t}$. Since
${\spurious{q}{\tau}{\dat(\K)} = \false}$, we have ${\tau(s) \approx \tau(t)
\in \datalogmodel}$. We have the following two cases.
\begin{itemize}
    \item Assume that ${v \in V_\mathsf{aux}}$. Clearly, ${\setof{\tau(s),\,
    \tau(t)} \subseteq \aux{\datalogmodel}}$. By the construction of
    $\datalogmodel$, there exists ${n \in \nat}$ such that ${\tau(s) \approx
    \tau(t) \in \datalogmodel[n]}$ and ${\tau(t) \in \aux{\datalogmodel[n]}}$.
    By Lemma \ref{lemma:eq}, we have ${\tau(s) = \tau(t)}$. Thus, ${\pi(s) =
    \pi(t)}$ and ${\pi(s) \approx \pi(t) \in \lpmodel}$, as required.

    \item Assume that ${v \not\in V_\mathsf{aux}}$. Then, we have ${\tau(t)
    \not\in \aux{\datalogmodel}}$ and, by Lemma \ref{lemma:dat-embed}, we have
    ${\pi(s) \approx \pi(t) \in \lpmodel}$.
\end{itemize}

\inductivestep Let ${v \in V_\mathsf{aux}}$ be an arbitrary vertex, let ${s, t
\in \terms{q}}$ be arbitrary terms such that ${\gamma(s) = v = \gamma(t)}$,
and assume that $\tau(v)$ is of the form $o_{R,A}$. By the definition of
$\gamma$, we have that ${s \sim t}$. Since ${\spurious{q}{\tau}{\dat(\K)} =
\false}$, we have ${\tau(s) \approx \tau(t) \in \datalogmodel}$. Since ${v \in
V_\mathsf{aux}}$, we have ${\setof{\tau(v), \tau(s), \tau(t)} \subseteq
\aux{\datalogmodel}}$. Then, by the construction of \datalogmodel, some ${n
\in \nat}$ exists such that ${\setof{\tau(v) \approx \tau(s),\, \tau(s)
\approx \tau(t)} \subseteq \datalogmodel[n]}$ and ${\setof{\tau(s), \tau(t)}
\subseteq \aux{\datalogmodel[n]}}$. By Lemma \ref{lemma:eq}, we have ${\tau(v)
= \tau(s)}$ and ${\tau(s) = \tau(t)}$. Now let $v'$ be the unique vertex of
$G_q$ such that ${\tuple{v',v} \in E}$. By definition, ${\pi(s) =
f_{R,A}(\pi(v')) = \pi(t)}$. Also, by the definition of $\delta$, we have
${\delta(f_{R,A}(\pi(v'))) = o_{R,A}}$, so property M1 holds. By the
reflexivity of $\approx$, we have ${\pi(s) \approx \pi(t) \in \lpmodel}$, and
so property M2 holds, as required.
\end{proof}

We finally prove the soundness of our approach.

\begin{lemma}[Soundness]
    Let \lpmodel and \datalogmodel be the minimal Herbrand models of $\Xi(\K)$
    and $\dat(\K)$, respectively; let ${q = \exists \vec{y}.
    \psi(\vec{x},\vec{y})}$ be an arbitrary CQ; and let $\tau$ be an arbitrary
    substitution such that ${\tau(q) \subseteq \datalogmodel}$ and
    ${\spurious{q}{\tau}{\dat(\K)} = \false}$. Then,
    ${\sproj{\tau}{\vec{x}}(q) \in \lpmodel}$.
\end{lemma}

\begin{proof}
For $q$ and $\tau$ as specified in the lemma, let $\pi$ be the substitution
defined as specified just before Lemma \ref{lemma:pi}, and assume that
${\spurious{q}{\dat(\K)}{\tau} = \false}$. By definition, we have
${\sproj{\pi}{\vec{x}} = \sproj{\tau}{\vec{x}}}$. We next show that ${\pi(q)
\subseteq \lpmodel}$.

First, let $A(t)$ be an arbitrary unary atom of $q$, we show that ${A(\pi(t))
\in \lpmodel}$. By assumption, we have ${A(\tau(t)) \in \datalogmodel}$. By
Lemma \ref{lemma:dat-embed}, for each term ${w \in
\restricteddomain{\lpmodel}}$ such that ${\delta(w) = \tau(t)}$, we have that
${A(w) \in \lpmodel}$. By property M1 of Lemma \ref{lemma:pi}, we have
${A(\pi(t)) \in \lpmodel}$.

\smallskip

Second, let $R(t',t)$ be an arbitrary atom of $q$; we show that
${R(\pi(t'),\pi(t)) \in \lpmodel}$. By assumption, we have
${R(\tau(t'),\tau(t)) \in \datalogmodel}$. We distinguish the following two
cases.
\begin{enumerate}
    \item Assume that ${\tau(t) \not\in \aux{\datalogmodel}}$. By Lemma
    \ref{lemma:dat-embed}, for all terms ${w',w \in
    \restricteddomain{\lpmodel}}$ such that ${\delta(w') = \tau(t')}$ and
    ${\delta(w) = \tau(t)}$, we have ${R(w',w) \in \lpmodel}$. By property M1
    of Lemma \ref{lemma:pi}, we have ${R(\pi(t'),\pi(t)) \in \lpmodel}$.

    \item Assume that ${\tau(t) \in \aux{\datalogmodel}}$, and assume that
    $\tau(t)$ is of the form $o_{R,A}$. Furthermore, let $v'$ be the unique
    vertex of $G_q$ such that ${\tuple{v',\gamma(t)} \in E_q}$ and
    ${\gamma(t') = v'}$. By the definition of $\pi$, we have ${\pi(t) =
    f_{R,A}(\pi(v'))}$. Since ${\spurious{q}{\dat(\K)}{\tau} = \false}$, we
    have ${\tau(v') \approx \tau(t') \in \datalogmodel}$. Since $\approx$ is a
    congruence relation, we have ${R(\tau(v'),\tau(t)) \in \datalogmodel}$. By
    Lemma \ref{lemma:dat-embed}, for each term ${w' \in
    \restricteddomain{\lpmodel}}$ such that ${\delta(w') = \tau(v')}$, we have
    ${R(w',f_{R,A}(w')) \in \lpmodel}$. By property M1 of Lemma
    \ref{lemma:pi}, we have ${R(\pi(v'),f_{R,A}(\pi(v'))) \in \lpmodel}$, and
    by Property M2 of Lemma \ref{lemma:pi}, we have ${\pi(t') \approx \pi(v')
    \in \lpmodel}$. Therefore, we have ${R(\pi(t'),f_{R,A}(\pi(v'))) \in
    \lpmodel}$. \qedhere
\end{enumerate}
\end{proof}

\subsection{Proof of Theorem \ref{th:upperbound}}

\thupperbound*

\begin{proof}
First, we argue that we can compute relation $\sim$ in polynomial time. For
each term $u$, we can decide whether ${u \in \aux{\dat(\K)}}$ by checking
whether, for each term $u'$, we have that ${\dat(\K) \models u \approx u'}$
implies ${u' \not\in \indnames}$. Since the number of variables occurring in
each clause in \dat(\K) is bounded by a constant, this check can be performed
in polynomial time. Thus, we can evaluate in polynomial time the precondition
of the $(\mathsf{fork})$ rule. In addition, the size of relation $\sim$ is
bounded by $|\terms{q}|^2$, the rules used to compute it are monotonic, and
each inference can be applied in polynomial time, so we can compute $\sim$ in
polynomial time.

Second, we show that we can decide whether $q$ is aux-cyclic w.r.t.\ $\tau$ in
polynomial time. Since $\sim$ can be computed in polynomial time and the size
of $G_\mathsf{aux}$ is polynomially bounded by the number of terms occurring
in $q$, we can compute $G_\mathsf{aux}$ in polynomial time. Also, we can check
whether $G_\mathsf{aux}$ is a acyclic by searching for a topological ordering
of its vertexes in linear time \cite{DBLP:books/daglib/0023376}.

For the \np upper bound, according to Theorem \ref{th:main} checking whether
${\K \models \pi(q)}$ amounts to guessing a candidate answer $\tau$ for $q$ in
the minimal Herbrand model of \dat(\K) such that ${\sproj{\tau}{\vec{x}} =
\pi}$ and to checking that ${\spurious{q}{\dat(\K)}{\tau} = \false}$. Since
each clause in \dat(\K) has a bounded number of variables, the minimal
Herbrand model of \dat(\K) can be computed in polynomial time. By the first
two observations, we conclude that the whole process can be carried out in
nondeterministic polynomial time in the combined size of $\dat(\K)$ and $q$.

For the \ptime upper bound, consider a fixed \elho TBox \T and a fixed
conjunctive query $q$. For an arbitrary ABox \A, we can enumerate in
polynomial time all possible answers to $q$ in the minimal Herbrand model of
${\dat(\T) \cup \A}$. Also, we can filter out those answers that are spurious
in polynomial time. Finally, we just check whether $\pi$ occurs in the
remaining (certain) answers.
\end{proof}

}{}

\end{document}